\documentclass{article} %
\usepackage{iclr2027_conference, times}
\usepackage{natbib}

\usepackage{url}
\usepackage{amssymb, amsmath}

\usepackage{amsmath,amsfonts,bm}

\def\eqref#1{equation~\ref{#1}}
\def\Eqref#1{Equation~\ref{#1}}

\def\1{\bm{1}}

\def\rc{{\textnormal{c}}}

\def\rz{{\textnormal{z}}}

\def\rvw{{\mathbf{w}}}
\def\rvx{{\mathbf{x}}}

\def\vmu{{\bm{\mu}}}
\def\vtheta{{\bm{\theta}}}

\def\vo{{\bm{o}}}

\def\vs{{\bm{s}}}

\def\vv{{\bm{v}}}

\def\vx{{\bm{x}}}
\def\vy{{\bm{y}}}
\def\vz{{\bm{z}}}

\def\mA{{\bm{A}}}

\def\mD{{\bm{D}}}

\def\mSigma{{\bm{\Sigma}}}

\DeclareMathAlphabet{\mathsfit}{\encodingdefault}{\sfdefault}{m}{sl}
\SetMathAlphabet{\mathsfit}{bold}{\encodingdefault}{\sfdefault}{bx}{n}

\def\gG{{\mathcal{G}}}

\def\sA{{\mathbb{A}}}
\def\sB{{\mathbb{B}}}
\def\sC{{\mathbb{C}}}
\def\sD{{\mathbb{D}}}
\def\sF{{\mathbb{F}}}

\def\sI{{\mathbb{I}}}
\def\sJ{{\mathbb{J}}}
\def\sK{{\mathbb{K}}}
\def\sL{{\mathbb{L}}}
\def\sM{{\mathbb{M}}}
\def\sN{{\mathbb{N}}}

\def\sS{{\mathbb{S}}}

\def\sU{{\mathbb{U}}}
\def\sV{{\mathbb{V}}}

\def\sX{{\mathbb{X}}}

\def\sZ{{\mathbb{Z}}}

\newcommand{\E}{\mathbb{E}}

\newcommand{\R}{\mathbb{R}}

\DeclareMathOperator*{\argmin}{arg\,min}

\usepackage{booktabs}
\usepackage{algorithm}
\usepackage{algpseudocode}
\usepackage{makecell}
\usepackage{graphicx,wrapfig}
\usepackage{tikz}
\usetikzlibrary{calc,shapes.geometric}

\usepackage{amsthm}
\usepackage{hyperref}
\hypersetup{
  hidelinks,
  pdftitle={Two-Point Local Optimality in k-Means via Boundary-Point Screening},
  pdfauthor={Wenlong Lyu, Xujie Xiao, Yuheng Jia}
}
\theoremstyle{definition}
\newtheorem{definition}{Definition}[section]

\theoremstyle{plain}
\newtheorem{proposition}[definition]{Proposition}
\newtheorem{theorem}[definition]{Theorem}

\title{Two-Point Local Optimality in $k$-Means via Boundary-Point Screening}

\author{Wenlong Lyu, Xujie Xiao \& Yuheng Jia \\
School of Computer Science and Engineering\\
Southeast University\\
Nanjing, Jiangsu 211189, China \\
\texttt{\{lwl25,220262246,yhjia\}@seu.edu.cn}
}

\iclrfinalcopy

\begin{document}
\maketitle
\lhead{Preprint}

\begin{abstract}
Lloyd's algorithm and the discrete local (D-local) optimization method \citep{li2025modified} for $k$-means provide only weak local-optimality guarantees, and their solution quality remains sensitive to initialization.
In this paper, we introduce \textbf{$r$-point local optimality}, under which no reassignment of at most $r$ samples decreases the objective function, and focus on $r=2$.
The main computational obstacle is the $\mathcal{O}(n^2(k^2+d))$ cost of exhaustive two-point certification for $n$ samples in $d$ dimensions and $k$ clusters.
To address this challenge, we prove that (i) every improving two-point move of a D-local optimum must involve a cluster shared by both reassignments, and (ii) only certificate-defined boundary points can participate in an improving pair.
Exploiting this structure, we propose \emph{Boundary-Point-Screened Two-Point Local Search} (BPS-2PLS), which terminates at a two-point local optimum.
For fixed $k,d$ and nonvanishing cluster occupancy, the number $m$ of retained candidates satisfies $m=\mathcal O_{\mathbb P}(\log n)$ under i.i.d.\ sampling from a bounded-support distribution with bounded density or from a Gaussian mixture.
Across twelve benchmarks, BPS-2PLS attains the lowest available mean WCSS on ten.
In a subsampling study, screening retains $0.10\%$--$2.81\%$ of samples on average at the largest tested sizes.
The code is available at \url{https://github.com/lwl-learning/BPS-2PLS}.
\end{abstract}

\section{Introduction}
$k$-means \citep{macqueen1967some} is a canonical clustering method because of its simple geometric interpretation, with broad applications including image analysis \citep{jia2025calibrated,li2025sample,qi2024superpixel,achanta2012slic}, large-scale similarity search \citep{jegou2011product}, single-cell genomics \citep{chari2024biophysically}, climate science \citep{lee2024dynamical}, and speech processing \citep{ambilduke2025tower}.
$k$-means seeks a partition that minimizes the within-cluster sum of squares (WCSS).
Despite the simplicity of its objective, finding a global minimizer is NP-hard \citep{aloise2009nphardness}.
Practical methods therefore seek a suboptimal solution from a given initialization.

Lloyd's algorithm alternates between assignment and centroid updates.
At termination, the assignments and centroids are each optimal with the other fixed \citep{lloyd1982least}.
This partial optimality does not exclude a reassignment that decreases WCSS after centroid recomputation.
Hartigan-style methods \citep{hartigan1979algorithm,telgarsky2010hartigan} address this weakness, and \citet{li2025modified} formalize \emph{D-local optimality}, under which no one-sample reassignment decreases WCSS after centroid recomputation.

However, two reassignments that are individually non-improving can yield an improvement when applied simultaneously, because their centroid updates interact through shared clusters.
This motivates \emph{$r$-point local optimality}, which excludes improvements from simultaneously reassigning at most $r$ samples and ranges from D-locality at $r=1$ to global optimality at $r=n$. We focus on $r=2$, the first level beyond one-point relocation.

Exact certification of two-point local optimality appears expensive because direct enumeration over sample pairs and their destination clusters yields $\binom{n}{2}(k-1)^2$ candidates.
Nevertheless, we establish that, at a D-local optimum, every improving two-point move must involve a cluster shared by both sample reassignments, reducing all potentially improving moves to six directed configurations.
To further reduce sample-pair enumeration, we construct, for each sample, an upper bound on the gains of connected two-point moves involving it, following the safe-screening principle \citep{elghaoui2012safe,fercoq2015mind}.
A sample with a non-positive upper bound cannot participate in any improving pair and can therefore be safely discarded.
After initial Lloyd convergence, alternating Min-D-LO one-point descent with exact search over the retained samples yields \emph{Boundary-Point-Screened Two-Point Local Search} (BPS-2PLS), which terminates at a two-point local optimum.
With $m$ retained candidates and lazily computed pair distances, screening and search cost $\mathcal O(k^2d+nk\log n+m^2(d+k))$ after each one-point descent phase, without a full sample-pair cache.

The benefit of screening depends on how $m$ grows with $n$.
We localize retained samples near empirical cluster boundaries and prove $m=\mathcal O_{\mathbb P}(\log n)$ for fixed $k,d$ under bounded-support sampling with bounded density or Gaussian-mixture sampling.
The bounds hold uniformly over D-local states with nonvanishing cluster occupancy.
Complementing this theoretical analysis, a six-dataset subsampling experiment shows mean pruning rates above $97\%$ at the largest tested sizes.
Together, these results show that exact two-point local optimality can be attained efficiently on the evaluated problems.

Overall, our contributions are summarized as follows:
\begin{itemize}
    \item We introduce a hierarchy from D-local to global optimality and characterize improving moves through connected cluster-transfer graphs.
    \item We develop BPS-2PLS, combining exact two-point search with safe sample screening to certify two-point local optimality.
    \item We establish uniform logarithmic candidate-size bounds $m=\mathcal{O}_{\mathbb{P}}(\log n)$ for two common distribution classes and a nonvanishing cluster-occupancy condition.
\end{itemize}

\section{Related Work}

Sample-reassignment methods include Hartigan-type one-sample relocations \citep{hartigan1979algorithm,telgarsky2010hartigan} and D-LO/Min-D-LO, which exclude improving one-point moves at termination \citep{li2025modified}.
Beyond single-sample relocations, \citet{tarsitano2003computational} studied general two-sample reassignments under the Friedman--Rubin determinant criterion by scanning sample pairs and their destination clusters.
In contrast to \citet{tarsitano2003computational}, we study the standard WCSS objective through a point-cardinality local-optimality hierarchy, with exact search and safe screening for its complete two-point neighborhood.

Center- or cluster-based neighborhoods include J-Means centroid relocation \citep{hansen2001jmeans}, merge-and-split repartitioning \citep{nielsen2014further}, and center swaps with approximation guarantees \citep{arya2004local,kanungo2004local}. These operations may change many labels, whereas our neighborhoods are defined by the number of reassigned samples.

A complementary line of work reduces computation without changing the output of a prescribed algorithm.
For example, exact Lloyd accelerations use triangle inequalities and distance bounds to reduce evaluations \citep{elkan2003triangle,hamerly2010making,newling2016fast}.
Safe screening in sparse optimization uses certificates to discard inactive variables without changing the target solution \citep{elghaoui2012safe,fercoq2015mind}.
BPS-2PLS combines these two goals: it goes beyond single-sample moves, uses safe screening to preserve every improving two-point move, and certifies two-point local optimality at termination.

\section{Problem Setup and Preliminaries}
\label{sec:preliminaries}

Throughout this paper, deterministic vectors and matrices are bold italic, sets are blackboard bold, and scalar entries are italic. Random draws use upright letters. All unqualified vector norms are Euclidean.
Let $\sX=(\vx_i)_{i=1}^n$ be an indexed collection of samples in $\R^d$, and write $[n]=\{1,\ldots,n\}$ and $[k]=\{1,\ldots,k\}$.
A clustering is represented by a label vector $\vz\in[k]^n$, with cluster index sets $\sI_\ell(\vz)=\{i\in[n]:z_i=\ell\}$, $\ell\in[k]$.
We restrict attention to partitions with nonempty clusters and to reassignments whose resulting partitions also satisfy this condition.
For a fixed state $\vz$, we suppress the dependence on $\vz$ in the cluster statistics and write $n_\ell=|\sI_\ell|, \vs_\ell=\sum_{i\in\sI_\ell}\vx_i, \vmu_\ell=\vs_\ell / n_\ell$.
The $k$-means objective, also known as the within-cluster sum of squares (WCSS), is
\begin{equation}
\textstyle
F(\vz)
=
\sum_{\ell=1}^k\sum_{i\in\sI_\ell}\lVert\vx_i-\vmu_\ell\rVert^2
=
\sum_{i=1}^n\lVert\vx_i\rVert^2
-
\sum_{\ell=1}^k\frac{\lVert\vs_\ell\rVert^2}{n_\ell}.
\label{eq:kmeans-objective}
\end{equation}
A reassignment from $\vz$ to $\vz'$ is improving precisely when $F(\vz)-F(\vz')>0$.
For later use, define the squared sample--centroid and sample--sample distances
\begin{equation}
D_{i\ell}=\lVert\vx_i-\vmu_\ell\rVert^2, \qquad p_{ij}=\lVert\vx_i-\vx_j\rVert^2.
\label{eq:distance-notation}
\end{equation}
Lloyd's algorithm alternates between assigning each sample to a nearest centroid and recomputing each centroid as the mean of its assigned samples \citep{lloyd1982least}.
At convergence, the resulting clustering satisfies
\begin{equation}
D_{i,z_i}\leq D_{ib},
\qquad
\forall i\in[n],\quad \forall b\in[k].
\label{eq:lloyd-fixed-point}
\end{equation}
This gives \emph{partial optimality}: assignments and centroids are each optimal with the other fixed \citep{li2025modified}.
However, partial optimality does not account for the centroid changes caused by reassigning a sample and therefore need not imply local optimality of the partition.

To account for these centroid changes, consider a one-point move $u=(i:s\to a)$, with $s=z_i$, $a\ne s$, and $n_s>1$. Its exact gain after recomputing the two affected centroids is
\begin{equation}
g_i(a)=\frac{n_s}{n_s-1}D_{is}-\frac{n_a}{n_a+1}D_{ia}.
\label{eq:one-point-gain}
\end{equation}
The first term is the reduction in WCSS obtained by removing $\vx_i$ from cluster $s$, whereas the second term is the increase caused by inserting it into cluster $a$.
Following \citet{li2025modified}, this gives the following notion of local optimality.
\begin{definition}[D-local optimality]
\label{def:d-local}
A clustering $\vz$ is \emph{D-locally optimal} if $g_i(a)\le0$ for every feasible one-point move $(i:z_i\to a)$, with the affected centroids recomputed.
\end{definition}
Every D-locally optimal clustering satisfies the Lloyd fixed-point condition, whereas the converse need not hold.
However, jointly reassigning multiple samples may still improve WCSS, motivating the point-cardinality hierarchy in the next section.

\section{\texorpdfstring{$r$}{r}-Point Local Optimality}
\label{sec:r-point-locality}

We extend D-local optimality to joint label changes and identify the cluster interactions that can produce improvements.
Let $\sZ_k=\{\vz\in[k]^n:\sI_a(\vz)\neq\varnothing, \forall a\in[k]\}$ be the feasible label space.
For $\vz\in\sZ_k$, a \emph{simultaneous $q$-point reassignment} is an unordered collection of $q$ sample reassignments
\begin{equation}
\sM=\{(i_t:a_t\to b_t)\}_{t=1}^q, \qquad a_t=z_{i_t},\quad b_t\in[k]\setminus\{a_t\},
\label{eq:simultaneous-move}
\end{equation}
where $i_1,\ldots,i_q$ are distinct.
The resulting vector $\vz^{\sM}$ changes each $z_{i_t}$ to $b_t$ and leaves all other entries unchanged.
Feasibility requires the final simultaneous reassignment to satisfy $\vz^{\sM}\in\sZ_k$; all affected centroids are then recomputed.

For $\vz,\vz'\in\sZ_k$, let $d_{\mathrm H}(\vz,\vz')=|\{i\in[n]:z_i'\neq z_i\}|$ denote the differing labels, and define the exact-$q$ shell and radius-$r$ point-reassignment neighborhood by
\begin{equation}
\textstyle
\sN_{=q}(\vz)=\{\vz'\in\sZ_k:d_{\mathrm H}(\vz,\vz')=q\},
\qquad
\sN_{\le r}(\vz)=\bigcup_{q=1}^r\sN_{=q}(\vz).
\label{eq:r-point-neighborhood}
\end{equation}
For a feasible move $\sM$, define its gain by $G(\sM;\vz)=F(\vz)-F(\vz^{\sM})$.

\begin{definition}[$r$-point local optimality] \label{def:r-point-locality}
A clustering $\vz\in\sZ_k$ is \emph{$r$-point locally optimal} if
\begin{equation}
F(\vz)\leq F(\vz'), \qquad \text{for every }\vz'\in\sN_{\le r}(\vz).
\label{eq:r-point-locality}
\end{equation}
\end{definition}

\begin{proposition}[Hierarchy] \label{prop:r-hierarchy}
Let $\sL_r$ denote the set of all $r$-point local optima. For $1\leq s<r\leq n$,
\begin{equation}
\sN_{\le s}(\vz)\subsetneq\sN_{\le r}(\vz),
\qquad
\sL_r\subseteq\sL_s.
\label{eq:hierarchy-inclusions}
\end{equation}
Moreover, $1$-point local optimality coincides with D-local optimality, whereas $n$-point local optimality coincides with global optimality over $\sZ_k$.
\end{proposition}
The objective is a sum of cluster contributions, so reassignments can interact only through shared clusters.
This motivates the following graph representation.

\begin{definition}[Cluster-transfer graph]
\label{def:transfer-graph}
For a simultaneous move $\sM$, its \emph{cluster-transfer graph} $\gG(\sM)$ is the directed multigraph whose vertex set $\sV(\sM)$ consists of the affected clusters.
For each reassignment, the graph contains one edge $a_t\to b_t$, labeled by $i_t$.
\end{definition}

Distinct components affect disjoint cluster statistics, so their gains are additive, yielding the following characterization.
\begin{theorem}[Connected-move characterization]
\label{thm:connected-r-move}
A clustering $\vz$ is $r$-point locally optimal if and only if every feasible simultaneous move $\sM$ with $1\leq|\sM|\leq r$ whose transfer graph $\gG(\sM)$ is weakly connected is non-improving.
Consequently, if $\vz$ is already $(r-1)$-point locally optimal, then every improving exact-$r$ move must have a weakly connected transfer graph.
\end{theorem}

The gain decomposition and proofs are given in Appendix~\ref{app:rpoint-proofs}.
A connected graph with $q$ edges has at most $q+1$ vertices, so every connected $q$-point move affects at most $q+1$ clusters.

Connectedness alone need not reduce the worst-case order for general $r$ (Appendix~\ref{app:connected-shell-count}).
For $r=2$, every improving move at a D-local state affects at most three clusters, yielding six configurations whose destination choices admit exact optimization.
Given $\mD$ and the required $p_{ij}$, destination search costs $\mathcal O(n^2k)$ per pass, without requiring cluster balance.

\section{Exact Two-Point Local Search} \label{sec:exact-two-point-search}

At a D-local state $\vz$, enumerating sample pairs $i<j$ and their destination labels yields $\binom{n}{2}(k-1)^2$ candidates. By Theorem~\ref{thm:connected-r-move}, every improving pair $u=(i:s\to a)$, $v=(j:z_j\to b)$ is connected: $\{s,a\}\cap\{z_j,b\}\neq\varnothing$, where $s=z_i$.

Fixing $u=(i:s\to a)$ as the anchor gives the six partner configurations in Table~\ref{tab:two-point-configurations}, indexed by $\tau$. Write $G(u,v):=G(\{u,v\};\vz)$. Let $\sJ_\tau(u)$ and $\sD_\tau(u,j)$ be the feasible partner-index and destination sets, and $\sF(u)=\{\tau:\sJ_\tau(u)\ne\varnothing\}$.

\begin{proposition}[Exact WCSS changes under cluster updates]
\label{prop:cluster-update-expressions}
For feasible cluster updates, the exact WCSS changes after centroid recomputation can be expressed using pre-update sizes and distances as follows, with $i\ne j$ for two-sample updates:
\begin{align}
    R_\ell(i) &= \tfrac{n_\ell}{n_\ell - 1} D_{i \ell}, & \text{gain from removing } \vx_i \text{ from cluster } \ell, \\
    A_\ell(i) &= \tfrac{n_\ell}{n_\ell + 1} D_{i \ell}, & \text{cost of inserting } \vx_i \text{ into cluster } \ell, \\
    R_\ell(i, j) &= \tfrac{n_\ell (D_{i \ell} + D_{j \ell}) - p_{ij}}{n_\ell - 2}, & \text{gain from removing } \vx_i \text{ and } \vx_j \text{ from cluster } \ell, \\
    A_\ell(i, j) &= \tfrac{n_\ell (D_{i \ell} + D_{j \ell}) + p_{ij}}{n_\ell + 2}, & \text{cost of inserting } \vx_i \text{ and } \vx_j \text{ into cluster } \ell, \\
    T_\ell(i, j) &= D_{i \ell} - D_{j \ell} + \tfrac{p_{ij}}{n_\ell}, & \text{gain from replacing } \vx_i \text{ with } \vx_j \text{ in cluster } \ell.
\end{align}
\end{proposition}
The cluster contributions sum to the six gains in Table~\ref{tab:two-point-configurations} (proof in Appendix~\ref{app:pd2pls-details}).
For example, a forward chain removes $\vx_i$ from cluster $s$, replaces $\vx_j$ with $\vx_i$ in cluster $a$, and inserts $\vx_j$ into cluster $b$, giving $G(u,v)=R_s(i)+T_a(j,i)-A_b(j)$.

\begin{table}[t]
\centering
\footnotesize
\setlength{\tabcolsep}{3pt}
\caption{Search spaces for the partner $v=(j:z_j\to b)$ and exact gains for fixed $u=(i:s\to a)$.}
\label{tab:two-point-configurations}
\begin{tabular}{@{}cccc@{}}
\toprule
Configuration $\tau$ & Partner constraint & \makecell{Destination and\\feasibility} & Exact gain $G(u,v)$ \\
\midrule
Same source, same destination
& $j\in \sI_s\setminus\{i\}$
& $b=a$, $n_s>2$
& $R_s(i,j)-A_a(i,j)$ \\
\addlinespace
Same source, different destinations
& $j\in \sI_s\setminus\{i\}$
& $b\notin\{s,a\}$, $n_s>2$
& $R_s(i,j)-A_a(i)-A_{b}(j)$ \\
\midrule
Forward chain
& $j\in \sI_a$
& $b\notin\{s,a\}$, $n_s>1$
& $R_s(i)+T_a(j,i)-A_{b}(j)$ \\
\addlinespace
Swap
& $j\in \sI_a$
& $b=s$
& $T_s(i,j)+T_a(j,i)$ \\
\midrule
Different sources, same destination
& $j\notin \sI_s\cup \sI_a$
& {\raggedright $b=a$, $n_s>1$, $n_{z_j}>1$\par}
& $R_s(i)+R_{z_j}(j)-A_a(i,j)$ \\
\addlinespace
Backward chain
& $j\notin \sI_s\cup \sI_a$
& $b=s$, $n_{z_j}>1$
& $R_{z_j}(j)+T_s(i,j)-A_a(i)$ \\
\bottomrule
\end{tabular}
\end{table}

Once the required entries of $\mD$ and $p_{ij}$ are available, each gain takes $\mathcal O(1)$ time to evaluate, but unbalanced cluster sizes can still leave $\Theta(n^2k^2)$ candidates (Proposition~\ref{prop:connected-shell-size}). Only the \emph{same source, different destinations} and \emph{forward chain} configurations require optimizing a second destination $b$.

For a fixed anchor $u=(i:s\to a)$ and partner sample $j$, both gains depend on $b$ solely through $-A_b(j)$.
When the destination domain is nonempty, choose any minimizer
\[
b^\star(j;s,a)\in\argmin_{b\in[k]\setminus\{s,a\}}A_b(j).
\]
For each sample $j$, store an ordered \emph{Top-3 all-cluster list} $\sK_j$ of the $\min\{3,k\}$ smallest scores $A_b(j)$, ranking all labels including the current source.
The scores and all lists can be constructed from $\mD$ in $\mathcal O(nk)$ time.
Excluding at most two labels leaves a minimizing feasible label in $\sK_j\setminus\{s,a\}$, enabling constant-time destination queries and $\mathcal O(n^2k)$ anchor--partner evaluations per pass.

$\mathrm{ExactPairSearch}(\vz,\sI)$ restricts both moved indices to $\sI\subseteq[n]$ while using the full clustering statistics and destination domains. It returns the best positive gain and its move, or $(0,\varnothing)$. After initial Lloyd convergence, Pair-Driven Two-Point Local Search (PD-2PLS) alternates Min-D-LO one-point descent with $\mathrm{ExactPairSearch}(\vz,[n])$ (Appendix~\ref{app:two-point-algorithms}). Here the descent phase repeatedly accepts the feasible one-point move with the largest positive gain until D-locality is reached.

Throughout one- and two-point descent, $\{n_\ell,\vmu_\ell\}$ and $\mD$ are maintained across moves.
A one-point move changes two clusters, and a connected two-point move changes at most three.
The affected cluster counts and means are updated from the moved samples, and only the corresponding columns of $\mD$ are refreshed, costing $\mathcal O(nd)$ per accepted move.
A Min-D-LO gain scan costs $\mathcal O(nk)$, so one scan and its possible update cost $\mathcal O(nk+nd)$; no full $\mathcal O(ndk)$ distance recomputation is needed.

\section{Safe Boundary-Point Screening}
\label{sec:bps}

At a D-local state, anchored search still considers up to $n(n-1)$ ordered sample pairs.
Motivated by the boundary-point intuition, Boundary-Point Screening (BPS) constructs an upper bound $U_i$ on every feasible connected two-point gain involving $i$ and discards $i$ if $U_i\le0$.

Fix an anchor $u=(i:s\to a)$ with $s=z_i$.
For configuration $\tau$, the envelope key $e=e(\tau,u)$ in Table~\ref{tab:bps-envelope-families} records the cluster labels needed by the partner terms and the pivot; anchors with the same key share an envelope.

To construct this bound, we first remove the dependence on the partner destination $b$:
\begin{equation}
G(u,v)\le C_\tau(u)+Q_e(j)+\eta_e p_{ij}.
\label{eq:bps-separable-gain}
\end{equation}
Here $C_\tau(u)$ collects the anchor terms and $Q_e(j)$ the partner terms; their coefficients are listed in Table~\ref{tab:bps-separable-gains} in Appendix~\ref{app:bps-details}. Equality holds when $b$ is prescribed or chosen as $b^\star(j;s,a)$.

Next, introduce a pivot $\vo_e\in\R^d$ and write $\rho_e(h)=\lVert\vx_h-\vo_e\rVert$ for each sample $h$.
The triangle and reverse triangle inequalities, accounting for the sign of $\eta_e$, yield
\begin{equation}
\eta_e p_{ij}\le
\eta_e\rho_e(i)^2+2|\eta_e|\rho_e(i)\rho_e(j)+\eta_e\rho_e(j)^2.
\label{eq:bps-pivot-bound}
\end{equation}
Substituting this inequality into~\eqref{eq:bps-separable-gain}, we define the partner line and the resulting bound by
\begin{align}
L_{e,j}(\rho)&=2|\eta_e|\rho_e(j)\rho+Q_e(j)+\eta_e\rho_e(j)^2,\notag\\
G(u,v)&\le C_\tau(u)+\eta_e\rho_e(i)^2+L_{e,j}(\rho_e(i)).
\label{eq:bps-partner-line}
\end{align}
For a fixed key $e$, each partner $j$ defines an affine function of the scalar query radius $\rho$.

Finally, we maximize over the shared partner superset $\widetilde{\sJ}_e\supseteq\sJ_\tau(u)$ in Table~\ref{tab:bps-envelope-families} to obtain a bound depending only on the anchor and configuration:
\begin{equation}
U_\tau(u):=C_\tau(u)+\eta_e\rho_e(i)^2+H_e(\rho_e(i))\ge G(u,v), \quad \text{where} \quad H_e(\rho)=\max_{j\in\widetilde{\sJ}_e}L_{e,j}(\rho).
\label{eq:bps-type-upper-bound}
\end{equation}
The remaining computational issues are:
(i) How can $H_e$ be constructed and queried efficiently?
(ii) How can it be reused across anchors?
(iii) How should its pivot $\vo_e$ be chosen?

A line upper envelope can be constructed in $\mathcal O(|\widetilde{\sJ}_e|\log|\widetilde{\sJ}_e|)$ time and queried in $\mathcal O(\log|\widetilde{\sJ}_e|)$ time by binary search on its breakpoints~\citep{deberg2008computational}; construction details are in Appendix~\ref{app:bps-details}.

For each feasible key in Table~\ref{tab:bps-envelope-families}, the coefficients and pivot are fixed, so one envelope serves all corresponding anchors.
There are at most $6(k-1)n$ partner lines in total, and each envelope contains at most $n$ lines, giving $\mathcal O(nk\log n)$ total construction time.

\begin{table}[t]
\centering
\small
\setlength{\tabcolsep}{3pt}
\caption{Reusable envelope families for $u=(i:s\to a)$, with $s\neq a$.}
\label{tab:bps-envelope-families}
\begin{tabular}{cccc}
\toprule
Configuration & Pivot $\vo_e$ & Key $e$ & Partner superset $\widetilde{\sJ}_e$ \\
\midrule
{\raggedright Same source (both configurations)\par}
& $\vmu_s$ & $(\tau,s,a)$ & $\sI_s$ \\
\addlinespace
Forward chain / swap
& $(\vmu_s + \vmu_a) / 2$ & $(\tau,s,a)$ & $\sI_a$ \\
\addlinespace
{\raggedright Different sources, same destination\par}
& $\vmu_a$ & $(\tau,a)$ & $\{j:z_j\neq a,\ n_{z_j}>1\}$ \\
\addlinespace
Backward chain
& $\vmu_s$ & $(\tau,s)$ & $\{j:z_j\neq s,\ n_{z_j}>1\}$ \\
\bottomrule
\end{tabular}
\end{table}

Accordingly, we aggregate the configuration bounds into the sample certificate
\begin{equation}
U_i=\max_{a\in[k]\setminus\{z_i\}}\ \max_{\tau\in\sF(u)}U_\tau(u),
\qquad u=(i:z_i\to a).
\label{eq:bps-sample-upper-bound}
\end{equation}
Given the line coefficients and query radii, constructing the envelopes and computing all certificates takes $\mathcal O(nk\log n)$ time.
We therefore retain the boundary candidate set $\sB=\{i\in[n]:U_i>0\}$ and let $m=|\sB|$.
At a D-local state, every improving two-point move is connected, so both of its samples have positive certificates and belong to $\sB$.
Therefore, $\mathrm{ExactPairSearch}(\vz,\sB)$ preserves the maximum positive gain while restricting distance evaluation and search to retained pairs.

Figure~\ref{fig:bps-pivots} illustrates the pivot choices. For fixed radii,~\eqref{eq:bps-pivot-bound} is tight for opposing directions when $\eta_e>0$ and aligned directions when $\eta_e<0$.
The radii can be evaluated in constant time from $\mD$ and the squared centroid distance $\Delta_{s, a}=\lVert\vmu_s-\vmu_a\rVert^2$, computed in $\mathcal O(k^2d)$ time (Appendix~\ref{app:bps-details}).

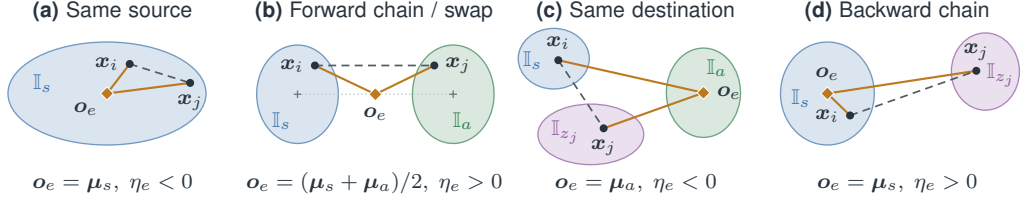
\begin{figure}[t]
\centering
\setlength{\parskip}{0pt}
\setlength{\abovecaptionskip}{4pt}
\definecolor{bpsInk}{HTML}{27333E}
\definecolor{bpsSource}{HTML}{356DA8}
\definecolor{bpsTarget}{HTML}{2F7D4B}
\definecolor{bpsOther}{HTML}{8A549E}
\definecolor{bpsPivot}{HTML}{BD7920}
\newcommand{\bpsFigCentre}[1]{%
  \draw[bpsInk!65,line width=.5pt]
    ($(#1)+(-.05,0)$)--($(#1)+(.05,0)$)
    ($(#1)+(0,-.05)$)--($(#1)+(0,.05)$);}
\begin{tikzpicture}[
  x=1cm,y=1cm,text=bpsInk,
  font=\fontsize{8}{9.5}\selectfont,
  source cloud/.style={draw=bpsSource!70,fill=bpsSource!18,line width=.5pt},
  target cloud/.style={draw=bpsTarget!70,fill=bpsTarget!18,line width=.5pt},
  other cloud/.style={draw=bpsOther!70,fill=bpsOther!18,line width=.5pt},
  sample/.style={circle,fill=bpsInk,inner sep=0pt,minimum size=3pt},
  pivot/.style={diamond,fill=bpsPivot,draw=white,line width=.25pt, inner sep=0pt,minimum size=5pt},
  radius/.style={draw=bpsPivot,line width=.8pt},
  pair distance/.style={draw=bpsInk!80,densely dashed,line width=.65pt},
  heading/.style={font=\sffamily\fontsize{8}{9.5}\selectfont, anchor=north,align=center,text width=3.25cm},
  every label/.style={inner sep=1.3pt},
  label distance=1pt
]
\begin{scope}
  \node[heading] at (1.625,2.7) {\textbf{(a)} Same source};
  \coordinate (ao) at (1.55,1.36);
  \coordinate (ai) at (1.85,1.75);
  \coordinate (aj) at (2.65,1.5);
  \path[source cloud] (ao) ellipse [x radius=1.31,y radius=.68];
  \node[text=bpsSource] at (.7,1.5) {$\sI_s$};
  \draw[radius] (ai)--(ao)--(aj);
  \draw[pair distance] (ai)--(aj);
  \node[sample,label=left:{$\vx_i$}] at (ai) {};
  \node[sample,label=below:{$\vx_j$}] at (aj) {};
  \node[pivot, label=below left:{$\vo_e$}] at (ao) {};
  \node at (1.625,.18) {$\vo_e=\vmu_s, \; \eta_e < 0$};
\end{scope}

\begin{scope}[shift={(3.46,0)}]
  \node[heading] at (1.625,2.7) {\textbf{(b)} Forward chain / swap};
  \coordinate (bs) at (.61,1.35);
  \coordinate (ba) at (2.67,1.35);
  \coordinate (bo) at ($(bs)!.5!(ba)$);
  \coordinate (bi) at (.84,1.73);
  \coordinate (bj) at (2.42,1.73);
  \path[source cloud] (bs) ellipse [x radius=.54,y radius=.65];
  \path[target cloud] (ba) ellipse [x radius=.54,y radius=.65];
  \node[text=bpsSource] at (0.4,1) {$\sI_s$};
  \node[text=bpsTarget] at (2.8,1) {$\sI_a$};
  \draw[bpsInk!30,densely dotted,line width=.5pt] (bs)--(ba);
  \bpsFigCentre{bs}\bpsFigCentre{ba}
  \draw[radius] (bi)--(bo)--(bj);
  \draw[pair distance] (bi)--(bj);
  \node[sample,label=left:{$\vx_i$}] at (bi) {};
  \node[sample,label=right:{$\vx_j$}] at (bj) {};
  \node[pivot, label=below:{$\vo_e$}] at (bo) {};
  \node at (1.625,.18) {$\vo_e=(\vmu_s + \vmu_a) / 2, \; \eta_e > 0$};
\end{scope}

\begin{scope}[shift={(6.92,0)}]
  \node[heading] at (1.625,2.7) {\textbf{(c)} Same destination};
  \coordinate (co) at (2.52,1.37);
  \coordinate (ci) at (.6,1.8);
  \coordinate (cj) at (1.2,.9);
  \path[source cloud] (.54,1.82) ellipse [x radius=.47,y radius=.4];
  \path[other cloud] (1.03,.82) ellipse [x radius=.7,y radius=.4];
  \path[target cloud] (co) ellipse [x radius=.49,y radius=.59];
  \node[text=bpsSource] at (.26,1.8) {$\sI_s$};
  \node[text=bpsOther] at (.70,.85) {$\sI_{z_j}$};
  \node[text=bpsTarget] at (2.7,1.7) {$\sI_a$};
  \draw[radius] (ci)--(co)--(cj);
  \draw[pair distance] (ci)--(cj);
  \node[sample,label=above:{$\vx_i$}] at (ci) {};
  \node[sample,label=below:{$\vx_j$}] at (cj) {};
  \node[pivot, label=right:{$\vo_e$}] at (co) {};
  \node at (1.625,.18) {$\vo_e=\vmu_a, \; \eta_e < 0$};
\end{scope}

\begin{scope}[shift={(10.38,0)}]
  \node[heading] at (1.625,2.7) {\textbf{(d)} Backward chain};
  \coordinate (do) at (.70,1.36);
  \coordinate (di) at (1.0,1.07);
  \coordinate (dj) at (2.67,1.66);
  \path[source cloud] (do) ellipse [x radius=.62,y radius=.68];
  \path[other cloud] (2.79,1.65) ellipse [x radius=.45,y radius=.51];
  \node[text=bpsSource] at (.35,1.3) {$\sI_s$};
  \node[text=bpsOther] at (3.0,1.6) {$\sI_{z_j}$};
  \draw[radius] (di)--(do)--(dj);
  \draw[pair distance] (di)--(dj);
  \node[sample,label=left:{$\vx_i$}] at (di) {};
  \node[sample,label=above:{$\vx_j$}] at (dj) {};
  \node[pivot, label=above:{$\vo_e$}] at (do) {};
  \node at (1.625,.18) {$\vo_e=\vmu_s, \; \eta_e > 0$};
\end{scope}

\end{tikzpicture}
\caption{Schematic pivot choices for the six connected configurations: (a) both same-source configurations; (b) forward chains and swaps; (c) different sources with the same destination; and (d) backward chains. Shaded regions denote clusters.}
\label{fig:bps-pivots}
\end{figure}

BPS-2PLS uses the same initial Lloyd run and Min-D-LO descent, followed by screening and $\mathrm{ExactPairSearch}(\vz,\sB)$ (Algorithm~\ref{alg:bps2pls}); strict descent ensures finite termination at a two-point local optimum.
Screening does not require evaluating any $p_{ij}$.
Candidate distances are computed in blocks in $\mathcal O(m^2d)$ time and reused across destination choices, giving $\mathcal O(m^2(d+k))$ pair-search time.
For $T_L$ Lloyd iterations, $T_1$ one-point gain scans, and $T_2$ screening-and-search rounds, the total cost is
\begin{equation}
\mathcal O\!\left(T_Lndk+T_1(nd+nk)+T_2\bigl[nd+nk\log n+k^2d+m^2(d+k)\bigr]\right),
\label{eq:bps-total-cost}
\end{equation}
where $m$ is the maximum retained count over the $T_2$ rounds. The scan and round counts include unsuccessful termination checks; the $nd$ term in each round covers an accepted pair's local state update. Appendix~\ref{app:distance-policies} gives the bound with varying candidate counts and compares lazy evaluation with full caching.

\section{Candidate-Set Size Guarantees}
\label{sec:candidate-size}

We formalize the boundary-point intuition by locating retained samples where the squared distances to the nearest and second-nearest empirical centers are close, and then derive logarithmic candidate-size bounds under two sampling models.

For a D-local state $\vz$ with $k\ge2$, write $n_{\min}=\min_\ell n_\ell$ and $p_{\max}=\max_{i,j\in[n]}p_{ij}$, the minimum cluster size and the squared sample diameter.
Let $D_{(1)}(\vx;\vz)\le D_{(2)}(\vx;\vz)$ be the two smallest values of $\lVert\vx-\vmu_\ell(\vz)\rVert^2$, counting ties.
Suppressing $\vz$ when it is fixed, define the boundary region
\[
\textstyle
\sU(\vz)
:=\left\{\vx\in\R^d:
D_{(2)}(\vx)-D_{(1)}(\vx)\le\frac{25p_{\max}}{n_{\min}}\right\}.
\]
\begin{proposition}[Geometric localization of retained samples]
\label{prop:candidate-near-localization}
Let $\vz$ be a D-local state with $n_{\min}\ge4$.
For the exact-arithmetic BPS certificate defined in Section~\ref{sec:bps}, the candidate set $\sB(\vz)$  satisfies
\begin{equation}
\sB(\vz)
\subseteq\left\{i\in[n]:\vx_i\in\sU(\vz)\right\}.
\label{eq:candidate-near-localization}
\end{equation}
\end{proposition}

\begin{wrapfigure}{r}{0.3\textwidth}
\centering
\definecolor{localizationBand}{RGB}{212,225,244}
\begin{tikzpicture}
\node[anchor=south west,inner sep=0] (localization) at (0,0)
  {\includegraphics[width=\linewidth]{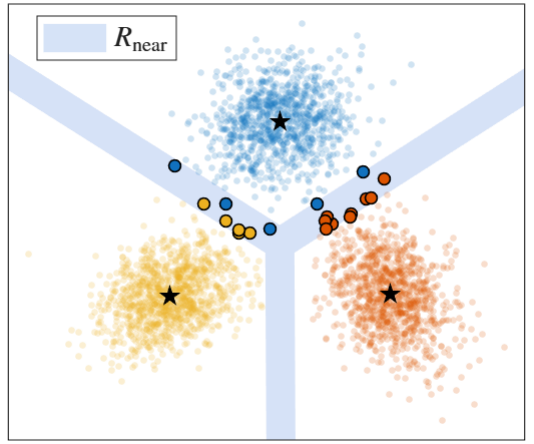}};
\begin{scope}[x={(localization.south east)},y={(localization.north west)}]
  \fill[white] (.045,.855) rectangle (.36,.972);
  \draw[gray!65,fill=white,line width=.25pt] (.05,.86) rectangle (.37,.965);
  \filldraw[draw=gray!65,fill=localizationBand,line width=.25pt]
    (.065,.88) rectangle (.18,.945);
  \node[anchor=west,inner sep=0,font=\fontsize{6}{7}\selectfont]
    at (.20,.9125) {$\sU(\vz)$};
\end{scope}
\end{tikzpicture}
\end{wrapfigure}
For distinct empirical centers, $\sU(\vz)$ is a union of portions of at most $\binom{k}{2}$ slabs around pairwise equal-distance hyperplanes.
The illustration on the right shows a two-dimensional Gaussian mixture, in which only 20 of the 3,000 samples lie in $\sU(\vz)$.
The proof is given in Appendix~\ref{app:candidate-localization}.

Localization alone does not guarantee a small candidate set: bounded, balanced two-cluster instances can have $m=n$ (Appendix~\ref{app:candidate-worst-case}). To obtain a sublinear bound, we control the boundary-region probability mass and squared sample diameter under nonvanishing empirical cluster occupancy.

\begin{theorem}[Logarithmic candidate-set size]
\label{thm:candidate-logarithmic}
Let $\rvx_1,\ldots,\rvx_n$ be i.i.d.\ random vectors from a fixed distribution $P$ on $\R^d$, and let $\vx_1,\ldots,\vx_n$ denote their realizations.
Fix $k,d$ and $\alpha\in(0,1/k]$.
Define the admissible state class
\begin{equation}
\sZ_{n,\alpha}
=\{\vz\in\sZ_k:\ \vz\text{ is D-local and }\min_\ell n_\ell(\vz)\ge\alpha n\},
\label{eq:candidate-state-class}
\end{equation}
and write $m(\vz)=|\sB(\vz)|$.
Suppose $P$ satisfies either of the following conditions:
\begin{enumerate}
    \item $P$ is supported on a fixed ball of radius $R$ centered at $\vx_0$ and has a Lebesgue density $f$ with $\lVert f\rVert _\infty<\infty$.
    \item $P$ is a Gaussian mixture $P=\sum_{\ell=1}^{k}\pi_\ell\mathcal N(\vtheta_\ell,\mSigma_\ell)$ with $\pi_\ell>0,\ \sum_{\ell=1}^{k}\pi_\ell=1,\ \mSigma_\ell\succ0$.
\end{enumerate}
The mixture parameters $(\pi_\ell,\vtheta_\ell,\mSigma_\ell)$ are fixed population quantities, whereas $\vmu_\ell(\vz)$ are empirical cluster means.
For the exact-arithmetic BPS certificate defined in Section~\ref{sec:bps}, set $\overline m_n=0$ if $\sZ_{n,\alpha}$ is empty. Otherwise,
\begin{equation}
\overline m_n:=\sup_{\vz\in\sZ_{n,\alpha}}m(\vz)
=\mathcal O_{\mathbb P}(\log n).
\label{eq:candidate-logarithmic}
\end{equation}
\end{theorem}

Here $\mathcal O_{\mathbb P}$ denotes order in probability~\citep{vandervaart1998asymptotic}. The proof combines uniform slab counting, diameter control, and empirical center separation (Appendices~\ref{app:candidate-counting}--\ref{app:candidate-distributions}). The bound covers data-dependent clusterings; the mixture case allows unequal weights and anisotropic covariances without requiring separated component means.

For fixed $d,k$, if a sequence of D-local outputs satisfies the occupancy condition with probability tending to one, then $m/n\to0$ in probability and the screening-and-search cost after one-point descent, including lazy distance evaluation, is $\mathcal O_{\mathbb P}(n\log n)$.
The theorem neither guarantees occupancy attainment nor bounds the number of Lloyd iterations, one-point scans, or two-point rounds.

\section{Experiments}
\suppressfloats[t]
We evaluate clustering quality and cost, heuristic enhancement, and candidate-set scaling. Datasets, methods, metrics, and timing protocols are detailed in Appendix~\ref{app:experimental-setup}.

\subsection{Comparison of Sample-Reassignment-Based Methods}
\label{sec:quality-cost}

We compare six sample-reassignment methods using ten shared initializations, with each method--initialization pair timed five times (Appendix~\ref{app:experimental-setup}).
In Table~\ref{tab:quality-cost}, BPS-2PLS attains the lowest available mean WCSS on ten of twelve datasets, including ties.
Hartigan yields lower WCSS on MNIST-2000 and Sensorless; however, neighborhood inclusion does not impose an ordering on independently reached local optima.

Table~\ref{tab:quality-cost} shows that BPS-2PLS reduces mean WCSS relative to Lloyd by $4.30\%$ on Yale and $7.28\%$ on ORL.
On five of the six largest datasets, the reductions are below $0.06\%$, including the CIFAR-10 tie; the reduction on Sensorless is $1.11\%$.
The diminishing influence of individual samples on larger clusters offers a possible explanation (Appendix~\ref{app:additional-experiments}).

In terms of computational cost, BPS-2PLS takes $1.08$--$2.57\times$ as long as D-LO on eight datasets and less time on the other four.
Hartigan takes $11.09\times$ and $242.76\times$ as long as BPS-2PLS on MNIST-2000 and Sensorless, respectively.
Attributing speedups specifically to screening requires paired BPS-on/off timings.

\begin{table}[t]
\centering
\caption{Mean WCSS and runtime (seconds) over ten shared initializations. Bold: lowest available WCSS. Dashes mark Hartigan runs omitted because of excessive runtime.}
\label{tab:quality-cost}
\scriptsize
\setlength{\tabcolsep}{3pt}
\begin{tabular}{llrrrrrr}
\toprule
Dataset & Metric & Lloyd & lite-KM & Hartigan & D-LO & Min-D-LO & BPS-2PLS \\
\midrule
Yale & WCSS & 3328.242 & 3328.242 & 3197.414 & 3214.216 & 3214.216 & \textbf{3185.016} \\
 & Time & 0.011 & 0.009 & 0.356 & 0.219 & 0.246 & 0.327 \\
\addlinespace[2pt]
ORL & WCSS & 2902.673 & 2902.673 & 2723.878 & 2735.73 & 2735.73 & \textbf{2691.304} \\
 & Time & 0.025 & 0.018 & 1.481 & 0.935 & 1.055 & 2.403 \\
\addlinespace[2pt]
USPS-1000 & WCSS & 22400.15 & 22400.15 & 22234.94 & 22254.42 & 22254.42 & \textbf{22229.35} \\
 & Time & 0.057 & 0.020 & 4.910 & 1.035 & 1.193 & 1.280 \\
\addlinespace[2pt]
COIL20 & WCSS & 32974.89 & 32974.89 & 32871.45 & 32783.05 & 32783.05 & \textbf{32772.27} \\
 & Time & 0.152 & 0.067 & 7.821 & 1.155 & 1.385 & 1.253 \\
\addlinespace[2pt]
MNIST-2000 & WCSS & 76682.46 & 76682.46 & \textbf{76439.91} & 76501.68 & 76501.68 & 76499.59 \\
 & Time & 0.242 & 0.087 & 27.96 & 2.013 & 2.345 & 2.521 \\
\addlinespace[2pt]
COIL100 & WCSS & 74693.18 & 74693.18 & --- & 73837.04 & 73837.04 & \textbf{73642.79} \\
 & Time & 1.305 & 0.854 & --- & 38.85 & 42.93 & 63.79 \\
\addlinespace[2pt]
20News & WCSS & 13421.12 & 13421.12 & 13422.92 & 13418.46 & 13418.47 & \textbf{13418.44} \\
 & Time & 1.640 & 0.756 & 391.89 & 35.25 & 16.82 & 21.26 \\
\addlinespace[2pt]
Sensorless & WCSS & 1641518 & 1641518 & \textbf{1622398} & 1629887 & 1629887 & 1623231 \\
 & Time & 0.938 & 0.608 & 1020.68 & 10.48 & 2.802 & 4.205 \\
\addlinespace[2pt]
CIFAR-10 & WCSS & \textbf{101000000} & \textbf{101000000} & --- & \textbf{101000000} & \textbf{101000000} & \textbf{101000000} \\
 & Time & 8.080 & 1.886 & --- & 22.59 & 24.51 & 13.94 \\
\addlinespace[2pt]
CIFAR-100 & WCSS & 85660060 & 85660060 & --- & 85610702 & 85610702 & \textbf{85610581} \\
 & Time & 19.67 & 6.664 & --- & 550.56 & 619.29 & 624.69 \\
\addlinespace[2pt]
MNIST & WCSS & 2751844 & 2751844 & --- & \textbf{2751804} & \textbf{2751804} & \textbf{2751804} \\
 & Time & 21.26 & 4.130 & --- & 100.81 & 113.23 & 176.45 \\
\addlinespace[2pt]
Fashion-MNIST & WCSS & 2242899 & 2242899 & --- & \textbf{2242898} & \textbf{2242898} & \textbf{2242898} \\
 & Time & 19.68 & 3.562 & --- & 65.00 & 50.51 & 39.87 \\
\bottomrule
\end{tabular}
\end{table}

\begin{figure}[t]
\centering
\includegraphics[width=\linewidth]{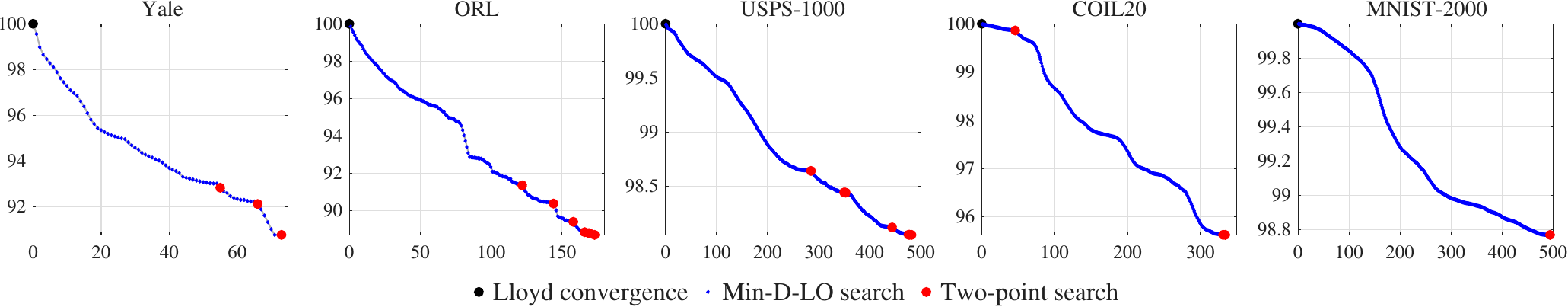}
\caption{BPS-2PLS descent trajectories on five datasets. The $x$-axis shows the iteration count after Lloyd convergence, and the $y$-axis shows WCSS as a percentage of the WCSS at Lloyd convergence.}
\label{fig:lossiter}
\end{figure}

Figure~\ref{fig:lossiter} shows how two-point moves extend descent beyond one-point local optima.
On Yale, ORL, USPS-1000, and COIL20, successful two-point moves are followed by further one-point improvements.
For example, on COIL20, the first D-local solution retains $99.87\%$ of the initial Lloyd WCSS, whereas the final solution retains $95.62\%$; most of this additional reduction comes from one-point moves made possible by the two-point updates.
The displayed MNIST-2000 run instead terminates at its first D-local solution, where two-point search finds no improvement.

\subsection{Enhancing Heuristic Methods}
\label{sec:heuristic-enhancement}

Table~\ref{tab:heuristic-enhancement} evaluates terminal (\emph{Post}) and internal (\emph{Inner}) BPS-2PLS refinement; insertion points are detailed in Appendix~\ref{app:heuristic-refinements}.
Among the four heuristic methods, Post reduces mean WCSS in all eighteen available comparisons.
For LocalSearch, Post reduces mean WCSS by $38.12$--$41.87\%$; this refinement includes centroid recomputation and one-point descent, which are absent from the baseline.
Inner lowers mean WCSS relative to Baseline in twelve of thirteen available comparisons but does not consistently outperform Post.
For J-Means, Inner uses $37.6\%$ and $40.2\%$ of the baseline runtime on Yale and ORL, while Post achieves lower mean WCSS.
For MergeSplit, Inner takes $1.56$--$5.32\times$ as long as Post and gives lower mean WCSS on Yale, ORL, and MNIST-2000; on COIL20, it has higher mean WCSS than both Baseline and Post.

\begin{table}[t]
\centering
\caption{Heuristic enhancement: mean WCSS and runtime (seconds). Bold: lowest WCSS within each method and dataset, including ties; dashes: omitted runs.}
\label{tab:heuristic-enhancement}
\scriptsize
\setlength{\tabcolsep}{2pt}
\begin{tabular}{ll*{5}{rr}}
\toprule
Method & Variant & \multicolumn{2}{c}{Yale} & \multicolumn{2}{c}{ORL} & \multicolumn{2}{c}{USPS-1000} & \multicolumn{2}{c}{COIL20} & \multicolumn{2}{c}{MNIST-2000} \\
\cmidrule(lr){3-4}\cmidrule(lr){5-6}\cmidrule(lr){7-8}\cmidrule(lr){9-10}\cmidrule(lr){11-12}
 &  & WCSS & Time & WCSS & Time & WCSS & Time & WCSS & Time & WCSS & Time \\
\midrule
Min-D-LO & Baseline & 3214.216 & 0.246 & 2735.73 & 1.055 & 22254.42 & 1.193 & 32783.05 & 1.385 & 76501.68 & 2.345 \\
 & Inner & \textbf{3185.016} & 0.327 & \textbf{2691.304} & 2.403 & \textbf{22229.35} & 1.280 & \textbf{32772.27} & 1.253 & \textbf{76499.59} & 2.521 \\
\addlinespace[2pt]
LocalSearch & Baseline & 5204.958 & 0.074 & 4594.401 & 0.196 & 35930.89 & 0.106 & 51511.93 & 0.439 & 123605.5 & 0.352 \\
 & Post & \textbf{3180.593} & 0.422 & \textbf{2670.928} & 2.439 & \textbf{22235.22} & 1.003 & \textbf{31799.29} & 2.329 & \textbf{76382.11} & 3.156 \\
\addlinespace[2pt]
MultiSwap & Baseline & 3347.979 & 0.045 & 2839.742 & 0.178 & 22363.17 & 0.090 & 33014.32 & 0.356 & 76669.54 & 0.488 \\
 & Post & 3197.036 & 0.365 & \textbf{2674.842} & 2.915 & \textbf{22244.28} & 1.071 & \textbf{32723.21} & 1.895 & \textbf{76571.75} & 2.466 \\
 & Inner & \textbf{3196.986} & 0.374 & 2683.272 & 2.948 & \textbf{22244.28} & 1.065 & \textbf{32723.21} & 1.883 & \textbf{76571.75} & 2.426 \\
\addlinespace[2pt]
J-Means & Baseline & 3214.364 & 16.51 & 2630.903 & 465.05 & 22354.6 & 16.30 & --- & --- & --- & --- \\
 & Post & \textbf{3160.917} & 16.74 & \textbf{2607.935} & 464.91 & \textbf{22237.06} & 17.27 & --- & --- & --- & --- \\
 & Inner & 3169.434 & 6.208 & 2608.329 & 186.88 & 22244.28 & 15.73 & --- & --- & --- & --- \\
\addlinespace[2pt]
MergeSplit & Baseline & 3191.945 & 0.791 & 2616.383 & 12.90 & 22242.59 & 3.220 & 31307.97 & 43.97 & 76494.29 & 96.93 \\
 & Post & 3167.338 & 1.023 & 2613.44 & 13.43 & \textbf{22236.35} & 3.403 & \textbf{31298.45} & 44.23 & 76493.35 & 97.25 \\
 & Inner & \textbf{3166.465} & 5.435 & \textbf{2608.782} & 71.40 & 22237.87 & 9.834 & 31324.05 & 94.50 & \textbf{76487.16} & 151.67 \\
\bottomrule
\end{tabular}
\end{table}

\subsection{Candidate-Set Size versus Sample Size}
\label{sec:candidate-scaling}

We study random subsets drawn from six datasets, recording $m$ at the first D-local state before any two-point move.
Figure~\ref{fig:candidate-scaling} summarizes five subsets per size, each averaged over three initializations; subset sizes and aggregation details are in Appendix~\ref{app:candidate-scaling-protocol}.

\begin{figure}[t]
\centering
\includegraphics[width=\linewidth]{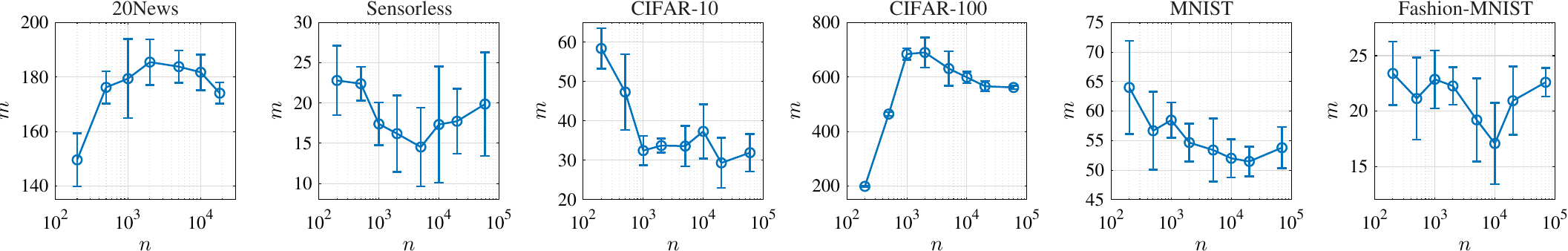}
\caption{Retained candidate count $m$ versus sample size $n$ at the first D-local state. Markers show means over five random subsets, each averaged over three initializations; error bars indicate one standard deviation across subsets.}
\label{fig:candidate-scaling}
\end{figure}

The candidate counts show no sustained increase across the evaluated range: 20News and CIFAR-100 rise over smaller sample sizes before leveling off or decreasing, while the other datasets are flat, decreasing, or non-monotone.
At the largest available size, mean retained fractions range from $0.10\%$ to $2.81\%$, corresponding to mean pruning rates above $97\%$ in all six cases.
These observations motivate studying sharper bounds for algorithmic outputs (Appendix~\ref{app:candidate-scaling-protocol}).

\section{Conclusion}

We introduced $r$-point local optimality and developed BPS-2PLS to certify two-point local optima of $k$-means.
Connected moves yield six configurations, while safe screening confines distance evaluation and search to retained samples.
Under the stated sampling and occupancy conditions, geometric localization gives $m=\mathcal O_{\mathbb P}(\log n)$ and $\mathcal O_{\mathbb P}(n\log n)$ screening-and-search time after one-point descent for fixed $d,k$.
Experiments show that BPS-2PLS enables further descent beyond one-point local optima and can improve existing heuristics when used as a refinement.

Screening can retain all samples, and the theorem bounds neither one-point descent time nor the number of rounds. Future work includes larger neighborhoods, faster screening and descent, and sharper candidate-size bounds for algorithmic outputs.

\subsection*{AI use statement}
In this work, we used generative AI tools solely to assist with grammar, spelling, and sentence clarity.
We have reviewed all AI-assisted edits and assume full responsibility for the final manuscript.
The generative AI contributed in no other way to this work.
The research idea, experimental design, and all other substantive content were developed and completed by the authors.

\subsection*{Ethics statement}
We confirm that our work adheres to the ICLR Code of Ethics.
This work uses publicly available datasets that contain no personally identifiable information.
Given the strictly technical nature of this contribution, we foresee no ethical issues or potential misuse.

\subsection*{Reproducibility statement}
Proofs of the theoretical results are provided in Appendices~\ref{app:rpoint-proofs}, \ref{app:connected-shell-count}, and~\ref{app:candidate-size}; detailed algorithms and screening constructions are given in Appendices~\ref{app:pd2pls-details} and~\ref{app:bps-details}.
Appendix~\ref{app:experimental-setup} specifies the dataset sources and representations, comparison methods and search budgets, initialization protocol, evaluation metrics, timing procedure, and computing environment.
Appendices~\ref{app:heuristic-refinements}, ~\ref{app:candidate-scaling-protocol} and~\ref{app:runtime-phases} describe the heuristic refinement insertion points and candidate-set scaling protocol.

\bibliography{ref}

@inproceedings{macqueen1967some,
  title     = {Some Methods for Classification and Analysis of Multivariate Observations},
  author    = {MacQueen, James B.},
  booktitle = {Proceedings of the Fifth Berkeley Symposium on Mathematical Statistics and Probability},
  volume    = {1},
  pages     = {281--297},
  publisher = {University of California Press},
  address   = {Berkeley, CA},
  year      = {1967},
  url       = {https://digicoll.lib.berkeley.edu/record/113015}
}

@article{lloyd1982least,
  title={Least squares quantization in PCM},
  author={Lloyd, Stuart},
  journal={IEEE transactions on information theory},
  volume={28},
  number={2},
  pages={129--137},
  year={1982},
  publisher={IEEE},
  doi = {10.1109/TIT.1982.1056489}
}

@article{aloise2009nphardness,
  title={NP-hardness of Euclidean sum-of-squares clustering},
  author={Aloise, Daniel and Deshpande, Amit and Hansen, Pierre and Popat, Preyas},
  journal={Machine learning},
  volume={75},
  number={2},
  pages={245--248},
  year={2009},
  publisher={Springer},
  doi = {10.1007/s10994-009-5103-0}
}

@inproceedings{li2025modified,
  title     = {Modified {K}-Means Algorithm with Local Optimality Guarantees},
  author    = {Li, Mingyi and Metel, Michael R. and Takeda, Akiko},
  booktitle = {Proceedings of the 42nd International Conference on Machine Learning},
  series    = {Proceedings of Machine Learning Research},
  volume    = {267},
  pages     = {35656--35684},
  publisher = {PMLR},
  year      = {2025},
  url       = {https://proceedings.mlr.press/v267/li25bt.html}
}

@inproceedings{jia2025calibrated,
  title     = {Towards Calibrated Deep Clustering Network},
  author    = {Jia, Yuheng and Cheng, Jianhong and Liu, Hui and Hou, Junhui},
  booktitle = {The Thirteenth International Conference on Learning Representations},
  year      = {2025},
  url       = {https://openreview.net/forum?id=JvH4jDDcG3}
}

@inproceedings{li2025sample,
  title     = {Learning from Sample Stability for Deep Clustering},
  author    = {Li, Zhixin and Jia, Yuheng and Liu, Hui and Hou, Junhui},
  booktitle = {Proceedings of the 42nd International Conference on Machine Learning},
  series    = {Proceedings of Machine Learning Research},
  volume    = {267},
  pages     = {34904--34919},
  publisher = {PMLR},
  year      = {2025},
  url       = {https://proceedings.mlr.press/v267/li25am.html}
}

@article{qi2024superpixel,
  title={Superpixel graph contrastive clustering with semantic-invariant augmentations for hyperspectral images},
  author={Qi, Jianhan and Jia, Yuheng and Liu, Hui and Hou, Junhui},
  journal={IEEE Transactions on Circuits and Systems for Video Technology},
  volume={34},
  number={11},
  pages={11360--11372},
  year={2024},
  publisher={IEEE},
  doi = {10.1109/TCSVT.2024.3418610}
}

@article{chari2024biophysically,
  title={Biophysically interpretable inference of cell types from multimodal sequencing data},
  author={Chari, Tara and Gorin, Gennady and Pachter, Lior},
  journal={Nature Computational Science},
  volume={4},
  number={9},
  pages={677--689},
  year={2024},
  publisher={Nature Publishing Group US New York},
  doi = {10.1038/s43588-024-00689-2}
}

@article{lee2024dynamical,
  title={The dynamical footprint of year-round North American Weather Regimes},
  author={Lee, Simon H and Messori, Gabriele},
  journal={Geophysical Research Letters},
  volume={51},
  number={2},
  pages={e2023GL107161},
  year={2024},
  publisher={Wiley Online Library},
  doi = {10.1029/2023GL107161}
}

@inproceedings{ambilduke2025tower,
  title={From Tower to Spire: Adding the Speech Modality to a Translation-Specialist LLM.},
  author={Ambilduke, Kshitij and Peters, Ben and Sannigrahi, Sonal and Keshwani, Anil and Lam, Tsz Kin and Martins, Bruno and Martins, Andr{\'e} FT and Boito, Marcely Zanon},
  booktitle={EMNLP (Findings)},
  pages={19658--19673},
  year={2025},
  doi = {10.18653/v1/2025.findings-emnlp.1071},
  url = {https://aclanthology.org/2025.findings-emnlp.1071/}
}

@article{achanta2012slic,
  title={SLIC superpixels compared to state-of-the-art superpixel methods},
  author={Achanta, Radhakrishna and Shaji, Appu and Smith, Kevin and Lucchi, Aurelien and Fua, Pascal and S{\"u}sstrunk, Sabine},
  journal={IEEE transactions on pattern analysis and machine intelligence},
  volume={34},
  number={11},
  pages={2274--2282},
  year={2012},
  publisher={IEEE},
  doi = {10.1109/TPAMI.2012.120}
}

@article{jegou2011product,
  title   = {Product Quantization for Nearest Neighbor Search},
  author  = {J{\'e}gou, Herv{\'e} and Douze, Matthijs and Schmid, Cordelia},
  journal = {IEEE Transactions on Pattern Analysis and Machine Intelligence},
  volume  = {33},
  number  = {1},
  pages   = {117--128},
  year    = {2011},
  doi     = {10.1109/TPAMI.2010.57}
}

@inproceedings{telgarsky2010hartigan,
  title     = {Hartigan's Method: $k$-means Clustering without Voronoi},
  author    = {Telgarsky, Matus and Vattani, Andrea},
  booktitle = {Proceedings of the Thirteenth International Conference on Artificial Intelligence and Statistics},
  pages     = {820--827},
  volume    = {9},
  series    = {Proceedings of Machine Learning Research},
  year      = {2010}
}

@article{elghaoui2012safe,
  title   = {Safe Feature Elimination for the LASSO and Sparse Supervised Learning Problems},
  author  = {El Ghaoui, Laurent and Viallon, Vivian and Rabbani, Tarek},
  journal = {Pacific Journal of Optimization},
  volume  = {8},
  number  = {4},
  pages   = {667--698},
  year    = {2012}
}

@inproceedings{fercoq2015mind,
  title     = {Mind the Duality Gap: Safer Rules for the Lasso},
  author    = {Fercoq, Olivier and Gramfort, Alexandre and Salmon, Joseph},
  booktitle = {Proceedings of the 32nd International Conference on Machine Learning},
  pages     = {333--342},
  volume    = {37},
  series    = {Proceedings of Machine Learning Research},
  year      = {2015}
}

@article{hartigan1979algorithm,
  title   = {Algorithm {AS} 136: A {K}-Means Clustering Algorithm},
  author  = {Hartigan, J. A. and Wong, M. A.},
  journal = {Applied Statistics},
  volume  = {28},
  number  = {1},
  pages   = {100--108},
  year    = {1979},
  doi     = {10.2307/2346830}
}

@article{arya2004local,
  title   = {Local Search Heuristics for {k}-Median and Facility Location Problems},
  author  = {Arya, Vijay and Garg, Naveen and Khandekar, Rohit and Meyerson, Adam and
             Munagala, Kamesh and Pandit, Vinayaka},
  journal = {SIAM Journal on Computing},
  volume  = {33},
  number  = {3},
  pages   = {544--562},
  year    = {2004},
  doi     = {10.1137/S0097539702416402}
}

@article{kanungo2004local,
  title   = {A Local Search Approximation Algorithm for {k}-Means Clustering},
  author  = {Kanungo, Tapas and Mount, David M. and Netanyahu, Nathan S. and
             Piatko, Christine D. and Silverman, Ruth and Wu, Angela Y.},
  journal = {Computational Geometry},
  volume  = {28},
  number  = {2--3},
  pages   = {89--112},
  year    = {2004},
  doi     = {10.1016/j.comgeo.2004.03.003}
}

@book{vandervaart1998asymptotic,
  title     = {Asymptotic Statistics},
  author    = {van der Vaart, A. W.},
  series    = {Cambridge Series in Statistical and Probabilistic Mathematics},
  publisher = {Cambridge University Press},
  year      = {1998},
  doi       = {10.1017/CBO9780511802256}
}

@article{hansen2001jmeans,
  title     = {{J-Means}: A New Local Search Heuristic for Minimum Sum of Squares Clustering},
  author    = {Hansen, Pierre and Mladenovi{\'c}, Nenad},
  journal   = {Pattern Recognition},
  volume    = {34},
  number    = {2},
  pages     = {405--413},
  year      = {2001},
  doi       = {10.1016/S0031-3203(99)00216-2}
}

@article{nielsen2014further,
  title   = {Further Heuristics for {$k$}-Means: The Merge-and-Split Heuristic and the {$(k,l)$}-Means},
  author  = {Nielsen, Frank and Nock, Richard},
  journal = {arXiv preprint arXiv:1406.6314},
  year    = {2014},
  url     = {https://arxiv.org/abs/1406.6314}
}

@inproceedings{elkan2003triangle,
  title     = {Using the Triangle Inequality to Accelerate {$k$}-Means},
  author    = {Elkan, Charles},
  booktitle = {Proceedings of the Twentieth International Conference on Machine Learning},
  pages     = {147--153},
  year      = {2003},
  url       = {https://cseweb.ucsd.edu/~elkan/kmeansicml03.pdf}
}

@inproceedings{hamerly2010making,
  title     = {Making {$k$}-Means Even Faster},
  author    = {Hamerly, Greg},
  booktitle = {Proceedings of the 2010 SIAM International Conference on Data Mining},
  pages     = {130--140},
  publisher = {SIAM},
  year      = {2010},
  doi       = {10.1137/1.9781611972801.12}
}

@inproceedings{newling2016fast,
  title     = {Fast {$k$}-Means with Accurate Bounds},
  author    = {Newling, James and Fleuret, Fran{\c{c}}ois},
  booktitle = {Proceedings of the 33rd International Conference on Machine Learning},
  series    = {Proceedings of Machine Learning Research},
  volume    = {48},
  pages     = {936--944},
  publisher = {PMLR},
  year      = {2016},
  url       = {https://proceedings.mlr.press/v48/newling16.html}
}

@article{tarsitano2003computational,
  title   = {A Computational Study of Several Relocation Methods for K-Means Algorithms},
  author  = {Tarsitano, Agostino},
  journal = {Pattern Recognition},
  volume  = {36},
  number  = {12},
  pages   = {2955--2966},
  year    = {2003},
  doi     = {10.1016/S0031-3203(03)00190-0}
}

@book{deberg2008computational,
  author    = {de Berg, Mark and Cheong, Otfried and van Kreveld, Marc and Overmars, Mark},
  title     = {Computational Geometry: Algorithms and Applications},
  edition   = {3},
  publisher = {Springer},
  year      = {2008},
  doi       = {10.1007/978-3-540-77974-2}
}

@misc{cai2011litekmeans,
  author = {Cai, Deng},
  title  = {{Litekmeans}: The Fastest {MATLAB} Implementation of Kmeans},
  year   = {2011},
  note   = {Software},
  url    = {http://www.zjucadcg.cn/dengcai/Data/Clustering.html}
}

@inproceedings{lattanzi2019better,
  title     = {A Better {$k$}-Means++ Algorithm via Local Search},
  author    = {Lattanzi, Silvio and Sohler, Christian},
  booktitle = {Proceedings of the 36th International Conference on Machine Learning},
  series    = {Proceedings of Machine Learning Research},
  volume    = {97},
  pages     = {3662--3671},
  publisher = {PMLR},
  year      = {2019},
  url       = {https://proceedings.mlr.press/v97/lattanzi19a.html}
}

@inproceedings{beretta2023multiswap,
  title     = {{Multi-Swap} {$k$}-Means++},
  author    = {Beretta, Lorenzo and Cohen-Addad, Vincent and Lattanzi, Silvio and Parotsidis, Nikos},
  booktitle = {Advances in Neural Information Processing Systems},
  volume    = {36},
  year      = {2023},
  url       = {https://papers.nips.cc/paper_files/paper/2023/hash/52d63f9e4b81f866bf69fb3c834aad47-Abstract-Conference.html}
}
\bibliographystyle{plainnat}

\appendix
\section{Proofs for \texorpdfstring{$r$}{r}-Point Local Optimality}
\label{app:rpoint-proofs}

\subsection{Exact Simultaneous-Reassignment Gain}

To analyze interactions among the reassignments in a simultaneous move $\sM$, we first define the changes in the size and sum of each cluster $\ell$ by
\begin{equation}
\Delta n_\ell(\sM) = \bigl|\{t:b_t=\ell\}\bigr|-\bigl|\{t:a_t=\ell\}\bigr|,
\qquad
\Delta \vs_\ell(\sM) = \sum_{t:b_t=\ell}\vx_{i_t}-\sum_{t:a_t=\ell}\vx_{i_t}.
\label{eq:cluster-statistic-changes}
\end{equation}
Let $\sV(\sM)$ be the set of affected clusters. \Eqref{eq:kmeans-objective} then yields the following exact gain.

\begin{proposition}
\label{prop:general-gain}
For every feasible simultaneous reassignment $\sM$,
\begin{equation}
G(\sM;\vz)
=
\sum_{\ell\in\sV(\sM)}
\left(
\frac{\lVert \vs_\ell+\Delta \vs_\ell(\sM)\rVert^2}{n_\ell+\Delta n_\ell(\sM)}
-
\frac{\lVert \vs_\ell\rVert^2}{n_\ell}
\right).
\label{eq:exact-simultaneous-gain}
\end{equation}
\end{proposition}

\begin{proof}[Proof of Proposition~\ref{prop:general-gain}]
For each cluster $\ell$, the move changes its size and sum to
\[
n_\ell^{\sM}=n_\ell+\Delta n_\ell(\sM),
\qquad
\vs_\ell^{\sM}=\vs_\ell+\Delta \vs_\ell(\sM).
\]
Because $\sM$ is feasible, $n_\ell^{\sM}\geq1$ for every $\ell$.
By~\eqref{eq:kmeans-objective}, the common term $\sum_i\lVert\vx_i\rVert^2$ cancels in $F(\vz)-F(\vz^{\sM})$, as do the contributions of all untouched clusters.
Substituting the updated statistics of the affected clusters yields~\eqref{eq:exact-simultaneous-gain}.
\end{proof}

\subsection{Componentwise Additivity and Connected Moves}

\begin{proposition}[Componentwise additivity]
\label{prop:component-additivity}
Let $\sM$ be feasible and let $\sM_1,\ldots,\sM_h$ be the moves induced by the weakly connected components of $\gG(\sM)$.
Then every $\sM_j$ is feasible when applied directly to $\vz$, and $G(\sM;\vz)=\sum_{j=1}^hG(\sM_j;\vz)$.
\end{proposition}

\begin{proof}[Proof of Proposition~\ref{prop:component-additivity}]
Different weakly connected components of $\gG(\sM)$ have disjoint vertex sets and therefore modify disjoint sets of cluster sufficient statistics.
For any cluster $\ell$ touched by component $\sM_j$, no transfer from another component can enter or leave $\ell$, so
\[
\Delta n_\ell(\sM_j)=\Delta n_\ell(\sM),
\qquad
\Delta \vs_\ell(\sM_j)=\Delta \vs_\ell(\sM).
\]
Since the full move is feasible, every touched cluster has positive final size under its component move as well, while all other clusters remain unchanged; hence each $\sM_j$ is feasible.
Partitioning the sum in~\eqref{eq:exact-simultaneous-gain} over the disjoint component vertex sets gives
\[
G(\sM;\vz)=\sum_{j=1}^hG(\sM_j;\vz).
\]
\end{proof}

\begin{proof}[Proof of Theorem~\ref{thm:connected-r-move}]
If $\vz$ is $r$-point locally optimal, every connected feasible move of size at most $r$ is non-improving by definition.
Conversely, assume that every such connected move is non-improving and consider an arbitrary feasible move $\sM$ with $|\sM|\leq r$.
If $\gG(\sM)$ is connected, then $G(\sM;\vz)\leq0$ by assumption.
Otherwise, let $\sM_1,\ldots,\sM_h$ denote its component moves.
By Proposition~\ref{prop:component-additivity}, each $\sM_j$ is feasible and connected, with $|\sM_j|\leq r$, and therefore $G(\sM_j;\vz)\leq0$.
Hence
\[
G(\sM;\vz)=\sum_{j=1}^hG(\sM_j;\vz)\leq0.
\]
Thus no feasible move involving at most $r$ samples improves the objective, proving $r$-point local optimality.

For the final statement, suppose $\vz$ is $(r-1)$-point locally optimal and an improving exact-$r$ move $\sM$ is disconnected.
Every component of $\gG(\sM)$ contains at most $r-1$ edges and is therefore non-improving.
Componentwise additivity would then imply $G(\sM;\vz)\leq0$, a contradiction.
\end{proof}

\section{Counting Connected Exact-\texorpdfstring{$r$}{r} Moves}
\label{app:connected-shell-count}

Connectedness removes structurally redundant moves, but it does not by itself make general exact-$r$ certification inexpensive.
Let
\begin{equation}
\sC_{=r}(\vz) = \{\vz^{\sM}\in\sN_{=r}(\vz):\gG(\sM)\text{ is weakly connected}\}
\label{eq:connected-shell}
\end{equation}
denote the connected exact-$r$ shell.

\begin{proposition}[Size of the connected exact-$r$ shell] \label{prop:connected-shell-size}
Fix $r\geq2$. In general,
\begin{equation}
|\sC_{=r}(\vz)| \leq \binom{n}{r}(k-1)^r.
\label{eq:connected-upper-trivial}
\end{equation}
This order is attainable up to $r$-dependent constants: a cluster of size $\Theta(n)$ can yield $\Theta_r(n^r k^r)$ connected moves.
If instead the partition is $\beta$-balanced, meaning $n_a\leq \beta n/k$ for every $a$, then
\begin{equation}
|\sC_{=r}(\vz)| = \mathcal{O}_r(\beta^r n^rk).
\label{eq:balanced-connected-upper}
\end{equation}
For equal-size balanced partitions with $k\geq r+1$ and $n/k\geq2$, this order is tight:
\begin{equation}
|\sC_{=r}(\vz)|=\Theta_r(n^rk).
\label{eq:balanced-connected-theta}
\end{equation}
\end{proposition}

\begin{proof}[Proof of Proposition~\ref{prop:connected-shell-size}]
The upper bound in~\eqref{eq:connected-upper-trivial} follows immediately from
$\sC_{=r}(\vz)\subseteq\sN_{=r}(\vz)$ and
$|\sN_{=r}(\vz)|\leq\binom{n}{r}(k-1)^r$.
To see that connectedness alone does not improve the worst-case order, suppose cluster $a$ has size $n_a\geq r+1$.
Choose any $r$ samples from $a$ and independently send each to any of the other $k-1$ clusters.
The source cluster remains nonempty, and all transfer edges share the source vertex $a$, so every such move is feasible and connected.
Consequently,
\[
|\sC_{=r}(\vz)|
\geq
\binom{n_a}{r}(k-1)^r.
\]
If $n_a=\Theta(n)$ and $r$ is fixed, this matches the upper bound up to $r$-dependent constants and gives the stated $\Theta_r(n^rk^r)$ order.

It remains to prove the balanced bound.
For an $r$-sample support $\sI\subseteq[n]$, let $s=|\{z_i:i\in\sI\}|$ be its number of distinct source clusters.
Fix $s\in\{1,\ldots,\min(r,k)\}$ and a set $\sA\subseteq[k]$ of $s$ source clusters.
Under $\beta$-balance,
\[
\sum_{a\in \sA}n_a\leq s\beta\frac{n}{k}.
\]
Hence the number of $r$-sample supports whose source clusters are contained in $\sA$ is at most
\[
\binom{\sum_{a\in \sA}n_a}{r}
\leq
\frac{1}{r!}\left(s\beta\frac{n}{k}\right)^r.
\]
Summing over the at most $\binom{k}{s}$ choices of $\sA$ shows that the number of supports with source diversity $s$ is
\begin{equation}
\mathcal{O}_r\!\left(\beta^r n^r k^{s-r}\right).
\label{eq:appendix-s-source-support-count}
\end{equation}

Now fix one such support.
A connected graph with $r$ edges has at most $r+1$ vertices, so in addition to its $s$ source clusters it can contain at most $r+1-s$ destination-only clusters.
If exactly $t$ destination-only clusters are used, choose them in at most $\binom{k-s}{t}$ ways.
Once these clusters are fixed, each destination lies among at most $s+t$ cluster labels.
Ignoring self-loop and connectivity restrictions only enlarges the count, so the number of destination assignments is at most
\[
\sum_{t=0}^{r+1-s}
\binom{k-s}{t}(s+t)^r
=
\mathcal{O}_r\!\left(k^{r+1-s}\right).
\]
Multiplying by~\eqref{eq:appendix-s-source-support-count} gives
$\mathcal{O}_r(\beta^r n^rk)$ connected candidates for each $s$.
Summing over at most $r$ possible values of $s$ proves~\eqref{eq:balanced-connected-upper}.

For sharpness, suppose all clusters have the same size $n/k\geq2$ and $k\geq r+1$.
Choose $r$ distinct source clusters, one sample from each, and one common destination cluster outside those sources.
Sending all selected samples to the common destination gives a feasible connected in-star.
The number of such moves is
\[
\binom{k}{r}\left(\frac{n}{k}\right)^r(k-r)
=
\Theta_r(n^rk),
\]
which matches the upper bound for $\beta=1$ and proves~\eqref{eq:balanced-connected-theta}.
\end{proof}

\section{Details for Exact Two-Point Local Search}
\label{app:pd2pls-details}

\subsection{Algorithms and Shared Search Routine}
\label{app:two-point-algorithms}
Algorithms~\ref{alg:pd2pls} and~\ref{alg:bps2pls} give the complete PD-2PLS and BPS-2PLS procedures in exact arithmetic.
Both run Lloyd initially and then alternate best-improvement one-point descent with a two-point search.
In the pseudocode, $\mathrm{Min\text{-}D\text{-}LO}$ denotes the one-point descent phase: scan all feasible moves using~\eqref{eq:one-point-gain}, accept a move with the largest positive gain, update the affected state, and repeat until no positive gain remains.
This phase reuses the supplied statistics and $\mD$ without restarting Lloyd.
This is the schedule used by \texttt{Model.myLocal2KM}; the standalone \texttt{Model.MinDLO\_KM} baseline resumes Lloyd after an accepted one-point move, and \texttt{Model.myLocal2KMv2} resumes Lloyd after an accepted pair.
Both call $\mathrm{ExactPairSearch}(\vz,\sI)$ (Algorithm~\ref{alg:exact-pair-search}), with $\sI=[n]$ for PD-2PLS and $\sI=\sB$ for BPS-2PLS.
The shared routine uses the counts and sample--centroid distances $\mD$ of the full current clustering; only the two moved sample indices are restricted to $\sI$.
Destination minima can be obtained using the Top-3 lists or the equivalent $\mathcal O(k)$ scans described in Appendix~\ref{app:sample-pair-ordering}.
The algorithms below compute sample-pair distances lazily, after $\sI$ is known, using the equivalent pair ordering in Appendix~\ref{app:sample-pair-ordering}. Each distance is reused across all destination choices for its pair; no full sample-pair cache is initialized.
Feasibility is checked jointly using Table~\ref{tab:two-point-configurations}, and cases with empty destination domains are skipped.

\begin{algorithm}[H]
\caption{Pair-Driven Two-Point Local Search (PD-2PLS)}
\label{alg:pd2pls}
\begin{algorithmic}[1]
\Require data $\sX$, cluster count $k$, initial clustering $\vz\in\sZ_k$
\State $(\vz,\mD,\{n_\ell,\vmu_\ell\})\gets\mathrm{Lloyd}(\sX,\vz)$
\While{true}
    \State $(\vz,\mD,\{n_\ell,\vmu_\ell\})
        \gets\mathrm{Min\text{-}D\text{-}LO}(\sX,\vz,\mD,\{n_\ell,\vmu_\ell\})$
    \State $(G_{\max},\sM_{\max})\gets\Call{ExactPairSearch}{\vz,[n]}$
    \If{$G_{\max}\leq0$}
        \State \Return $\vz$
    \EndIf
    \State Apply $\sM_{\max}$ to $\vz$; update affected $n_\ell,\vmu_\ell$ and columns of $\mD$
\EndWhile
\end{algorithmic}
\end{algorithm}

\begin{algorithm}[H]
\caption{Boundary-Point-Screened Two-Point Local Search (BPS-2PLS)}
\label{alg:bps2pls}
\begin{algorithmic}[1]
\Require data $\sX$, cluster count $k$, initial clustering $\vz\in\sZ_k$
\State $(\vz,\mD,\{n_\ell,\vmu_\ell\})\gets\mathrm{Lloyd}(\sX,\vz)$
\While{true}
    \State $(\vz,\mD,\{n_\ell,\vmu_\ell\})
        \gets\mathrm{Min\text{-}D\text{-}LO}(\sX,\vz,\mD,\{n_\ell,\vmu_\ell\})$
    \State Compute insertion scores $\mA$, Top-3 lists $\{\sK_j\}$,
        and center distances $\bm{\Delta}$
    \State Build the shared envelopes specified in Table~\ref{tab:bps-envelope-families}
    \State Compute $U_i$ for all $i\in[n]$ using~\eqref{eq:bps-sample-upper-bound}
    \State $\sB\gets\{i\in[n]:U_i>0\}$
    \State $(G_{\max},\sM_{\max})\gets\Call{ExactPairSearch}{\vz,\sB}$
    \If{$G_{\max}\leq0$}
        \State \Return $\vz$
    \EndIf
    \State Apply $\sM_{\max}$ to $\vz$; update affected $n_\ell,\vmu_\ell$ and columns of $\mD$
\EndWhile
\end{algorithmic}
\end{algorithm}

\begin{algorithm}[H]
\caption{$\mathrm{ExactPairSearch}(\vz,\sI)$}
\label{alg:exact-pair-search}
\begin{algorithmic}[1]
\Require data $\sX$, D-local clustering $\vz$, sample set $\sI\subseteq[n]$,
    full-clustering statistics and $\mD$
\State $G_{\max}\gets0$; $\sM_{\max}\gets\varnothing$
\ForAll{blocks covering unordered pairs $i<j$ with $i,j\in\sI$}
    \State Compute $p_{ij}=\lVert\vx_i-\vx_j\rVert^2$ for the current block
    \ForAll{pairs $i<j$ in the block}
        \State Find the best feasible connected move $\sM$ for this pair
            \Statex\hspace{\algorithmicindent}\hspace{\algorithmicindent}\hspace{\algorithmicindent}
            using Table~\ref{tab:two-point-configurations} and Appendix~\ref{app:sample-pair-ordering}
        \If{such a move exists and $G(\sM;\vz)>G_{\max}$}
            \State $G_{\max}\gets G(\sM;\vz)$; $\sM_{\max}\gets\sM$
        \EndIf
    \EndFor
    \State Discard the current distance block
\EndFor
\State \Return $(G_{\max},\sM_{\max})$
\end{algorithmic}
\end{algorithm}

\paragraph{Local state updates.}
For a one- or two-point move $\sM$, we use the pre-move statistics and the changes defined in Appendix~\ref{app:rpoint-proofs}.
For each affected cluster $\ell\in\sV(\sM)$, set
\[
n'_\ell=n_\ell+\Delta n_\ell(\sM),\qquad
\vmu'_\ell=\frac{n_\ell\vmu_\ell+\Delta\vs_\ell(\sM)}{n'_\ell},\qquad
D'_{i\ell}=\lVert\vx_i-\vmu'_\ell\rVert^2\quad(i\in[n]).
\]
All other counts, centroids, and columns of $\mD$ remain valid.
Updating the counts costs $\mathcal O(1)$, and updating the centroids costs $\mathcal O(d)$ for a one-point or connected two-point move; refreshing the two or at most three affected distance columns costs $\mathcal O(nd)$.
Locality refers to the affected columns, not to restricting their rows to moved samples. The same update is used inside Min-D-LO and after a pair is accepted.
Insertion scores, destination lists, and screening envelopes are refreshed from the current state before the next screening round; they must not be reused unchanged after centroid or count updates.

\subsection{Cluster-Update Primitives and Feasibility}
For a cluster index set $\sI_\ell$, recall its size $n_\ell$, sum $\vs_\ell$, and centroid $\vmu_\ell=\vs_\ell/n_\ell$.
Its within-cluster sum of squares is
\[
\operatorname{WCSS}(\sI_\ell)
=
\sum_{i\in \sI_\ell}\lVert \vx_i\rVert^2-\frac{\lVert \vs_\ell\rVert^2}{n_\ell}.
\]
Substituting the updated sufficient statistics after removing, inserting, or replacing samples gives
\begin{align*}
\operatorname{WCSS}(\sI_\ell)-\operatorname{WCSS}(\sI_\ell\setminus\{i\})
&=\frac{n_\ell}{n_\ell-1}D_{i\ell}=R_\ell(i),\\
\operatorname{WCSS}(\sI_\ell\cup\{i\})-\operatorname{WCSS}(\sI_\ell)
&=\frac{n_\ell}{n_\ell+1}D_{i\ell}=A_\ell(i),\\
\operatorname{WCSS}(\sI_\ell)-\operatorname{WCSS}(\sI_\ell\setminus\{i,j\})
&=\frac{n_\ell(D_{i\ell}+D_{j\ell})-p_{ij}}{n_\ell-2}=R_\ell(i,j),\\
\operatorname{WCSS}(\sI_\ell\cup\{i,j\})-\operatorname{WCSS}(\sI_\ell)
&=\frac{n_\ell(D_{i\ell}+D_{j\ell})+p_{ij}}{n_\ell+2}=A_\ell(i,j),\\
\operatorname{WCSS}(\sI_\ell)-\operatorname{WCSS}((\sI_\ell\setminus\{i\})\cup\{j\})
&=D_{i\ell}-D_{j\ell}+\frac{p_{ij}}{n_\ell}=T_\ell(i,j).
\end{align*}
Combining the contributions of the affected clusters yields the six gains in Table~\ref{tab:two-point-configurations}.

For $u=(i:z_i\to a)$, both same-source configurations require $n_{z_i}>2$. In a different-source/same-destination move, neither source is replenished, so $n_{z_i}>1$ and $n_{z_j}>1$ are required. A forward chain has $z_j=a$ and $b\notin\{z_i,a\}$: cluster $a$ receives $\vx_i$ while losing $\vx_j$, and only $n_{z_i}>1$ is needed. In a backward chain, $z_j\notin\{z_i,a\}$ and $b=z_i$, so only $n_{z_j}>1$ is required. A swap has $z_j=a$ and $b=z_i$ and preserves both source sizes. These are exactly the size conditions in Table~\ref{tab:two-point-configurations}. Any case requiring an unavailable third label is omitted.

\subsection{Exact Top-3 Destination Queries}
We use $\max\varnothing=-\infty$ and $\min\varnothing=+\infty$; the search routine clips its returned gain at zero, and ties may be broken arbitrarily.
The formula for $A_\ell(i)$ defines a ranking score for every $i,\ell$; its interpretation as an insertion cost requires $i\notin\sI_\ell$.
Computing the insertion-score matrix from $\mD$ costs $\mathcal O(nk)$ time, and selecting each sample's three smallest scores costs $\mathcal O(k)$, hence $\mathcal O(nk)$ in total.
Let $\sK_j$ contain the $\min\{3,k\}$ labels with the smallest insertion scores over all clusters, including $z_j$ in the ranking. A query excludes $\{z_i,a\}$, which contains $z_j$ in both configurations under consideration. If $k\leq3$, all labels are cached. If $k>3$, at most two of the three cached labels are excluded, so at least one remains. Whenever the admissible domain is nonempty, every uncached score is at least as large as each cached score; hence an admissible minimizer lies in $\sK_j$. This justifies constant-time destination queries using Top-3 lists, including ties. When the admissible domain is empty, the corresponding configuration is absent.
The screening implementation obtains these lists by sorting each row of $\mA$, which costs $\mathcal O(nk\log k)$ rather than linear-time selection. Since $k\le n$, this is covered by the $\mathcal O(nk\log n)$ screening bound.

\subsection{Equivalent Sample-Pair Ordering}
\label{app:sample-pair-ordering}
Algorithm~\ref{alg:exact-pair-search} regroups the anchored search of Section~\ref{sec:exact-two-point-search} by unordered sample pairs. Fix $i<j$ in $\sI$ and abbreviate their source labels as $s=z_i$ and $c=z_j$. If $s=c$, the best distinct-destination insertion cost is
\[
\min_{\substack{a,b\in[k]\setminus\{s\}\\a\neq b}}
\bigl(A_a(i)+A_b(j)\bigr).
\]
This minimum can be recovered from the constant-size product $(\sK_i\setminus\{s\})\times(\sK_j\setminus\{s\})$ subject to $a\neq b$.
To see this, start with an optimal feasible pair $(a,b)$. With $b$ fixed, the excluded labels for $i$ are $\{s,b\}$, so the Top-3 query yields $a'\in\sK_i\setminus\{s,b\}$ with $A_{a'}(i)\leq A_a(i)$.
Fixing $a'$ then yields $b'\in\sK_j\setminus\{s,a'\}$ with $A_{b'}(j)\leq A_b(j)$.
The pair $(a',b')$ remains feasible and lies in the cached product without increasing the total cost.
If no distinct pair exists, the minimum is $+\infty$.

For $s=c$, the two configuration optima are therefore
\[
\max_{a\neq s}\{R_s(i,j)-A_a(i,j)\},
\qquad
R_s(i,j)-\min_{\substack{a,b\in[k]\setminus\{s\}\\a\ne b}}\bigl(A_a(i)+A_b(j)\bigr),
\]
subject to $n_s>2$. If $s\neq c$, the common-destination optimum is
\[
\max_{a\notin\{s,c\}}
\{R_s(i)+R_c(j)-A_a(i,j)\},
\]
and the two chain orientations and swap have gains
\begin{align*}
R_s(i)+T_c(j,i)-\min_{b\notin\{s,c\}}A_b(j),\\
R_c(j)+T_s(i,j)-\min_{a\notin\{s,c\}}A_a(i),\\
T_s(i,j)+T_c(j,i),
\end{align*}
respectively, subject to the joint feasibility conditions above. Each chain query excludes the two source labels $\{s,c\}$, so the same Top-3 argument applies. Only the common-destination cases scan $\mathcal O(k)$ labels; all other destination optimizations take constant time. Given $p_{ij}$, destination optimization therefore costs $\mathcal O(k)$ per unordered pair. Computing $p_{ij}$ once adds $\mathcal O(d)$, rather than multiplying the distance cost by the number of destinations. This ordering returns the same maximum positive gain as anchored enumeration for every $\sI\subseteq[n]$.

The implemented gain kernel instead scans the admissible labels for each chain and, for a same-source pair, records the two smallest insertion costs for each sample during the common-destination scan. If the two best labels coincide, it compares the two choices that use one runner-up. This computes the same distinct-destination minimum without requiring precomputed Top-3 lists in the pair-search routine, and still takes $\mathcal O(k)$ time per pair.

\subsection{Exactness and Termination}
At a D-local state, every improving exact two-point move is connected by Theorem~\ref{thm:connected-r-move}. Choose either of its transfers as $u=(i:z_i\to a)$ and write the other as $v=(j:z_j\to b)$. If $z_j=z_i$, then $b=a$ or $b\notin\{z_i,a\}$. If $z_j=a$, then $b=z_i$ or $b\notin\{z_i,a\}$. If $z_j\notin\{z_i,a\}$, connectedness requires $b\in\{z_i,a\}$. These alternatives are exactly the six rows of Table~\ref{tab:two-point-configurations}, including their joint feasibility conditions.

Algorithm~\ref{alg:exact-pair-search} with $\sI=[n]$ visits every unordered sample pair. The destination optimizations in Appendix~\ref{app:sample-pair-ordering} cover all six connected configurations, including both chain orientations, and return a feasible move with gain at least that of any other connected move for the pair. Hence the maximum positive gain returned equals the maximum gain over $\sN_{=2}(\vz)$, clipped at zero. If it returns zero, no improving exact two-point move exists, and D-locality already excludes improving one-point moves. The clustering is therefore two-point locally optimal by Definition~\ref{def:r-point-locality}.

After an accepted pair, PD-2PLS (Algorithm~\ref{alg:pd2pls}) locally updates the affected cluster statistics and distance columns, then resumes Min-D-LO descent from that state. Every accepted one- or two-point move strictly decreases WCSS, and there are finitely many feasible clusterings. Hence the algorithm terminates. BPS-2PLS has the same guarantee because screening preserves every improving pair; its envelopes are rebuilt after one-point descent reaches the next D-local state.

\subsection{Distance Evaluation, Time, and Memory}
\label{app:distance-policies}
The sample--centroid cache $\mD$ and the sample--sample distances $p_{ij}$ have different roles. Lloyd initializes $\mD$ and the cluster statistics, which are then maintained through local updates during Min-D-LO and after accepted pairs. Screening uses these quantities, the insertion scores, and $\bm{\Delta}$; it does not evaluate any $p_{ij}$. In particular, the diameter $p_{\max}$ in Section~\ref{sec:candidate-size} is only an analysis quantity and is not computed by BPS. Sample--sample distances are needed only for exact gain evaluation after screening.

\paragraph{Full caching.}
One may compute all $p_{ij}$ once in $\mathcal O(n^2d)$ time and store them in $\mathcal O(n^2)$ space. They remain valid across clustering updates because the samples do not change. Each later lookup costs $\mathcal O(1)$, yielding $\mathcal O(k^2d+nk\log n+m^2k)$ screening-and-search time after one-point descent. The quadratic preprocessing and storage remain even when screening retains very few samples.

\paragraph{Lazy block evaluation.}
Algorithms~\ref{alg:pd2pls}--\ref{alg:exact-pair-search} instead compute distances only within the current search set. For BPS-2PLS this set is $\sB$, so at most $\binom m2$ distinct pair distances are required. A block with at most $w$ rows of retained samples and at most $m$ retained columns uses $\mathcal O(wm)$ distance storage. Processing the blocks, reusing each distance across all destination choices and discarding a block after use, costs $\mathcal O(m^2d)$ time. The resulting screening-and-search cost after one-point descent is
\begin{equation}
\mathcal O\bigl(k^2d+nk\log n+m^2(d+k)\bigr),\qquad m=|\sB|,
\label{eq:bps-round-cost}
\end{equation}
with no $\mathcal O(n^2d)$ preprocessing. Distances may be recomputed in later rounds; the per-round bound includes this work.

\begin{table}[htbp]
\centering
\small
\caption{Distance-policy costs for BPS-2PLS. Screening-and-search time excludes one-point descent and accepted-move state updates; peak memory includes the input data and clustering workspace. The block row limit is $w$.}
\label{tab:distance-policy-costs}
\setlength{\tabcolsep}{4pt}
\begin{tabular}{@{}lcc@{}}
\toprule
 & Full sample-pair cache & Lazy candidate-pair blocks \\
\midrule
Pair-distance preprocessing & $\mathcal O(n^2d)$ & None \\
Screening and search per round & $\mathcal O(k^2d+nk\log n+m^2k)$ & $\mathcal O(k^2d+nk\log n+m^2(d+k))$ \\
Peak memory & $\mathcal O(nd+nk+k^2+n^2)$ & $\mathcal O(nd+nk+k^2+wm)$ \\
\bottomrule
\end{tabular}
\end{table}

The common storage comprises the $nd$ input coordinates, the $nk$ sample--centroid and insertion-score entries, and the $k^2$ centroid-distance entries; storing all envelopes requires at most $\mathcal O(nk)$ additional space. A copied candidate-data array uses $\mathcal O(md)\le\mathcal O(nd)$ space, already covered in Table~\ref{tab:distance-policy-costs}. With fixed $w$, lazy evaluation avoids quadratic sample-pair storage even if $m=n$. The two policies evaluate the same gains and have the same exact-arithmetic optimality guarantee.

\paragraph{Total runtime.}
Let $T_L$ count Lloyd iterations, $T_1$ count complete one-point gain scans across all Min-D-LO phases, and $T_2$ count screening-and-search rounds. These counts include the unsuccessful scans and the final search used to certify termination. Initialization and construction of $\mD$ take $\mathcal O(ndk)$, absorbed into $\mathcal O(T_Lndk)$ since $T_L\ge1$. Each one-point scan costs $\mathcal O(nk)$ and accepts at most one move, whose local update costs $\mathcal O(nd)$. Each two-point round rebuilds the screening quantities in $\mathcal O(k^2d+nk\log n)$ time and, if successful, adds a local update costing $\mathcal O(nd)$. If $m_t$ samples survive in round $t$, the lazy implementation therefore takes
\begin{equation}
\mathcal O\!\left(T_Lndk+T_1(nd+nk)+\sum_{t=1}^{T_2}\bigl[nd+k^2d+nk\log n+m_t^2(d+k)\bigr]\right).
\label{eq:bps-total-cost-varying}
\end{equation}
Setting $m=\max_{1\le t\le T_2}m_t$ gives~\eqref{eq:bps-total-cost}; an average candidate count cannot replace this maximum in that worst-case bound. The $\mathcal O(nd)$ update allowance is unnecessary in unsuccessful scans or rounds but gives a uniform upper bound. These counts refer to individual scans and rounds, not to calls of a complete Min-D-LO solver; no convergence-rate bound on $T_L,T_1,T_2$ is asserted.

Without screening, PD-2PLS uses all $n$ samples and takes $\mathcal O(T_Lndk+T_1(nd+nk)+T_2n^2(d+k))$ time with lazy distances. Under full caching, the bound includes the one-time $\mathcal O(n^2d)$ preprocessing term, with $m_t^2(d+k)$ replaced by $m_t^2k$ in the BPS-2PLS bound. Under Theorem~\ref{thm:candidate-logarithmic}, the lazy-distance term is $\mathcal O_{\mathbb P}(d\log^2 n)$, so the screening-and-search cost after one-point descent, including an accepted pair's $\mathcal O(nd)$ update, remains $\mathcal O_{\mathbb P}(n\log n)$ for fixed $d,k$.

\section{Details for Safe Boundary-Point Screening}
\label{app:bps-details}

\subsection{Configuration Coefficients}
Table~\ref{tab:bps-separable-gains} gives the coefficients in~\eqref{eq:bps-separable-gain}, obtained from the gains in Table~\ref{tab:two-point-configurations} after optimizing any free partner destination.
The first four configurations use $e=(\tau,s,a)$, different sources with the same destination use $e=(\tau,a)$, and the backward chain uses $e=(\tau,s)$.

\begin{table}[t]
\centering
\footnotesize
\setlength{\tabcolsep}{3pt}
\caption{Separable gains for $u=(i:s\to a)$, with $s=z_i$ and envelope key $e=e(\tau,u)$.}
\label{tab:bps-separable-gains}
\begin{tabular}{cccc}
\toprule
Configuration $\tau$ & $C_\tau(u)$ & $Q_e(j)$ & $\eta_e$ \\
\midrule
Same source same destination
& $\frac{n_s}{n_s - 2} D_{is} - \frac{n_a}{n_a + 2} D_{ia}$
& $\frac{n_s}{n_s - 2} D_{js} - \frac{n_a}{n_a + 2} D_{ja}$
& $-\frac{1}{n_s - 2} - \frac{1}{n_a + 2}$ \\
\addlinespace
Same source, different destinations
& $\frac{n_s}{n_s - 2} D_{is} - A_a(i)$
& $\frac{n_s}{n_s - 2} D_{js} - A_{b^\star(j;s,a)}(j)$
& $-\frac{1}{n_s - 2}$ \\
\addlinespace
Forward chain
& $R_s(i) - D_{ia}$
& $D_{ja} - A_{b^\star(j;s,a)}(j)$
& $\frac{1}{n_a}$ \\
\addlinespace
Swap
& $D_{is}-D_{ia}$
& $D_{ja}-D_{js}$
& $\frac{1}{n_s}+\frac{1}{n_a}$ \\
\addlinespace
Different sources, same destination
& $R_s(i)-\frac{n_a}{n_a+2}D_{ia}$
& $R_{z_j}(j)-\frac{n_a}{n_a+2}D_{ja}$
& $-\frac{1}{n_a+2}$ \\
\addlinespace
Backward chain
& $D_{is}-A_a(i)$
& $R_{z_j}(j)-D_{js}$
& $\frac{1}{n_s}$ \\
\bottomrule
\end{tabular}
\end{table}

\subsection{Envelope Construction and Queries}
We use a standard construction for line upper envelopes~\cite{deberg2008computational}.
For a fixed key $e$, we sort the $|\widetilde{\sJ}_e|$ lines by increasing slope, retaining only the largest intercept among lines with equal slope.
We then scan the sorted lines while maintaining a stack of active lines and their intersection breakpoints, removing redundant lines until the breakpoints are strictly increasing.
Each line is inserted and removed at most once, so sorting and envelope construction take $\mathcal O(|\widetilde{\sJ}_e|\log|\widetilde{\sJ}_e|)$ time and $\mathcal O(|\widetilde{\sJ}_e|)$ space.
An individual query at $\rho_e(i)$ can use binary search on the breakpoints in $\mathcal O(\log|\widetilde{\sJ}_e|)$ time, with constant-time handling of envelopes with at most one line. The implementation instead sorts the query radii for each envelope and scans its breakpoints monotonically, then restores the original query order. With $q$ queries this takes $\mathcal O(q\log q+q+|\widetilde{\sJ}_e|)$ time; the total number of queries and lines across families is $\mathcal O(nk)$, so the same $\mathcal O(nk\log n)$ aggregate bound holds.

\subsection{Cached Pivot Radii}
The triangle and reverse triangle inequalities give
\[
|\rho_e(i)-\rho_e(j)|\le\lVert\vx_i-\vx_j\rVert\le\rho_e(i)+\rho_e(j).
\]
Using the upper bound for $\eta_e>0$ and the lower bound for $\eta_e<0$ yields~\eqref{eq:bps-pivot-bound}.
With these pivot choices, the radii needed for line construction and envelope queries can be computed from the cached distances $\mD$.
We first compute the matrix $\bm{\Delta}=(\Delta_{sa})$, where $\Delta_{sa}=\lVert\vmu_s-\vmu_a\rVert^2$, for all cluster pairs in $\mathcal O(k^2d)$ time.
For any sample index $h$, the required squared radius is
\begin{equation}
\rho_e(h)^2 = \begin{cases}
D_{h\ell}, & \vo_e=\vmu_\ell,\\[2pt]
\frac12D_{hs}+\frac12D_{ha}-\frac14\Delta_{sa}, & \vo_e=(\vmu_s+\vmu_a)/2.
\end{cases}
\label{eq:bps-midpoint-radius}
\end{equation}

\section{Proofs for Candidate-Set Size Guarantees}
\label{app:candidate-size}

We first prove Proposition~\ref{prop:candidate-near-localization} and give the construction with $m=n$.
We then bound the number of samples in $\sU(\vz)$ using a cover by slabs and prove Theorem~\ref{thm:candidate-logarithmic} for the two distribution classes.

\subsection{Deterministic Localization of Retained Samples}
\label{app:candidate-localization}

Fix a D-local state $\vz$ with $n_{\min}=\min_\ell n_\ell\ge4$, and let $p_{\max}=\max_{i,j\in[n]}p_{ij}$.
We analyze the exact-arithmetic certificate in~\eqref{eq:bps-sample-upper-bound} with the pivots and partner supersets in Table~\ref{tab:bps-envelope-families}.
The localization argument is deterministic.

\begin{proof}[Proof of Proposition~\ref{prop:candidate-near-localization}]
Let $g_i^{\max}=\max_{a\ne z_i}g_i(a)$ be the best one-point gain for sample $i$.
The condition $n_{\min}\ge4$ ensures that every constituent one-point move considered below is feasible, and D-locality makes its gain nonpositive.
We first bound the screening certificate by $g_i^{\max}+22p_{\max}/n_{\min}$, then translate this bound into a squared-distance margin.

For a cluster $\ell$, define the interactions
\begin{align}
I_\ell^-(i,j)
&=\frac{n_\ell(D_{i\ell}+D_{j\ell})}{(n_\ell-2)(n_\ell-1)}
  -\frac{p_{ij}}{n_\ell-2},\notag\\
I_\ell^+(i,j)
&=\frac{n_\ell(D_{i\ell}+D_{j\ell})}{(n_\ell+1)(n_\ell+2)}
  -\frac{p_{ij}}{n_\ell+2},\notag\\
J_\ell(i,j)
&=\frac{p_{ij}}{n_\ell}
  -\frac{D_{i\ell}}{n_\ell-1}
  -\frac{D_{j\ell}}{n_\ell+1}.
\label{eq:candidate-interactions}
\end{align}
The cluster-update identities in Proposition~\ref{prop:cluster-update-expressions} give
\[
I_\ell^-=R_\ell(i,j)-R_\ell(i)-R_\ell(j),\qquad
I_\ell^+=A_\ell(i)+A_\ell(j)-A_\ell(i,j),
\]
and $J_\ell(i,j)=T_\ell(i,j)-R_\ell(i)+A_\ell(j)$.
Therefore, for $u=(i:s\to a)$ and $v=(j:c\to b)$ with $s=z_i$ and $c=z_j$, each exact gain has the form
\begin{equation}
G(u,v)=g_i(a)+g_j(b)+I(u,v),
\label{eq:candidate-gain-decomposition}
\end{equation}
where Table~\ref{tab:candidate-interaction-bounds} lists $I(u,v)$.
When the partner destination is free, choose the optimizing destination from Section~\ref{sec:exact-two-point-search}; the identity still holds.

\begin{table}[ht]
\centering
\small
\caption{Interactions and their upper bounds after radial relaxation, for the anchor $i:s\to a$. The last column bounds the excess over the two one-point gains.}
\label{tab:candidate-interaction-bounds}
\begin{tabular}{@{}lll@{}}
\toprule
Configuration & Interaction $I(u,v)$ & Excess upper bound \\
\midrule
Same source, same destination
& $I_s^-(i,j)+I_a^+(i,j)$ & $(16/n_s+6/n_a)p_{\max}$ \\
Same source, different destinations
& $I_s^-(i,j)$ & $16p_{\max}/n_s$ \\
Forward chain
& $J_a(j,i)$ & $5p_{\max}/n_a$ \\
Swap
& $J_s(i,j)+J_a(j,i)$ & $5(1/n_s+1/n_a)p_{\max}$ \\
Different sources, same destination
& $I_a^+(i,j)$ & $6p_{\max}/n_a$ \\
Backward chain
& $J_s(i,j)$ & $5p_{\max}/n_s$ \\
\bottomrule
\end{tabular}
\end{table}

To verify the last column, recall that $p_{\max}=\max_{i,j\in[n]}p_{ij}$.
Since each empirical center is a cluster mean, convexity gives
\[
D_{i\ell}\le\frac1{n_\ell}\sum_{h\in \sI_\ell}p_{ih}\le p_{\max}.
\]
The condition $n_\ell\ge n_{\min}\ge4$ ensures
$(n_\ell-1)^{-1}\le2/n_\ell$ and $(n_\ell-2)^{-1}\le2/n_\ell$.
Using these inequalities and discarding negative terms in~\eqref{eq:candidate-interactions} gives
\[
I_\ell^-\le\frac{8p_{\max}}{n_\ell},\qquad
I_\ell^+\le\frac{2p_{\max}}{n_\ell},\qquad
J_\ell\le\frac{p_{\max}}{n_\ell}.
\]
Next, fix an envelope key $e$ and abbreviate $\rho_h=\rho_e(h)=\lVert\vx_h-\vo_e\rVert$.
A center pivot gives $\rho_h^2=D_{h\ell}\le p_{\max}$, while a midpoint pivot gives
\[
\rho_h^2=\left\lVert\vx_h-\frac{\vmu_s+\vmu_a}{2}\right\rVert^2
\le\frac{D_{hs}+D_{ha}}2\le p_{\max}.
\]
The radial replacement of $\eta_e p_{ij}$ has slack
\begin{align}
\epsilon_e(i,j)
&=\eta_e(\rho_i^2+\rho_j^2-p_{ij})+2|\eta_e|\rho_i\rho_j,\notag\\
0\le\epsilon_e(i,j)&\le4|\eta_e|\rho_i\rho_j\le4|\eta_e|p_{\max}.
\label{eq:candidate-radial-slack}
\end{align}
Indeed, $\rho_i^2+\rho_j^2-p_{ij}=2(\vx_i-\vo_e)^\top(\vx_j-\vo_e)$, so the bounds follow from Cauchy--Schwarz.
Using the coefficients $\eta_e$ in Table~\ref{tab:bps-separable-gains} and
$(n_\ell-2)^{-1}\le2/n_\ell$
yields the bound in every row of Table~\ref{tab:candidate-interaction-bounds}.
Each such bound is at most $22p_{\max}/n_{\min}$.

The same bounds cover the extra partners introduced by envelope sharing.
In the two same-source configurations, the self-line $j=i$ satisfies the same algebraic identities with $p_{ii}=0$.
Its two one-point terms are $g_i(a)$ and either $g_i(a)$ or $g_i(b)$, where $b\notin\{s,a\}$.
For different sources with the same destination, the enlarged set also includes $j\in\sI_s$; its partner term is $g_j(a)$.
For a backward chain, the extra partners have $j\in \sI_a$ and partner term $g_j(s)$.
All of these one-point gains are feasible and nonpositive.
Configurations requiring a third label are omitted if that label is unavailable.
Thus every relaxed gain $C_\tau(u)+\eta_e\rho_i^2+L_{e,j}(\rho_i)$ is at most $g_i^{\max}+22p_{\max}/n_{\min}$. Taking maxima over partners, configurations, and anchor destinations yields
\begin{equation}
U_i\le g_i^{\max}+\frac{22p_{\max}}{n_{\min}}.
\label{eq:candidate-margin-bridge}
\end{equation}
These algebraic checks are needed because a shared envelope also contains lines that do not represent feasible two-point moves.

If $i\in\sB(\vz)$ and $s=z_i$, there is an $a\ne s$ with
$g_i(a)=g_i^{\max}>-22p_{\max}/n_{\min}$.
The difference between this gain and the fixed-center squared-distance difference is
\begin{equation}
g_i(a)-(D_{is}-D_{ia})
=\frac{D_{is}}{n_s-1}+\frac{D_{ia}}{n_a+1}
\in[0,3p_{\max}/n_{\min}].
\label{eq:candidate-distance-correction}
\end{equation}
Applying~\eqref{eq:candidate-distance-correction} with any destination $b\ne s$, D-locality gives
\[
D_{is}-D_{ib}\le g_i(b)\le0.
\]
Thus the assigned center is a nearest center, and $D_{(1)}(\vx_i)=D_{is}$.
For the selected destination $a$, the lower gain bound and~\eqref{eq:candidate-distance-correction} give $D_{ia}-D_{is}<25p_{\max}/n_{\min}$.
Consequently,
\[
\begin{aligned}
D_{(2)}(\vx_i)-D_{(1)}(\vx_i)
&=\min_{b\ne s}(D_{ib}-D_{is})\\
&\le D_{ia}-D_{is}<\frac{25p_{\max}}{n_{\min}}.
\end{aligned}
\]
Hence $\vx_i\in\sU(\vz)$, proving Proposition~\ref{prop:candidate-near-localization}.
Taking the contrapositive, any sample outside this region has $U_i\le0$.
The inclusion imposes a necessary condition for retention, so samples inside the region may still be discarded.
\end{proof}

\subsection{A Worst-Case Construction}
\label{app:candidate-worst-case}

We construct a deterministic family of bounded, balanced D-local states with $m=n$.
For every even $N\ge4$, set $a^2=1/[2(N-1)]$ and form two clusters in $\R^2$:
\[
\sI_1:\quad \frac N2\text{ copies each of }(-a,1)\text{ and }(-a,-1),
\qquad
\sI_2:\quad \frac N2\text{ copies each of }(a,1)\text{ and }(a,-1).
\]
The instance is bounded, has $k=2$ and $n=2N$, and both cluster sizes equal $N$.
The centroids are $\vmu_1=(-a,0)^\top$ and $\vmu_2=(a,0)^\top$.
Each sample has squared distance $1$ to its own centroid and $1+4a^2$ to the other, so every one-point gain is
\[
\frac{N}{N-1}-\frac{N}{N+1}(1+4a^2)=0.
\]
The displayed state is therefore D-local.
For samples $\vx_i=(-a,1)^\top$ and $\vx_j=(a,-1)^\top$, the swap $u=(i:1\to2)$, $v=(j:2\to1)$ has gain
\[
G(u,v)
=2-2(1+4a^2)+\frac2N(4+4a^2)
=\frac4N>0.
\]
The same calculation applies with the second coordinates reversed.
Every sample participates in an improving swap, so any sample-level screening rule preserving all improving pairs must retain all $n$ samples.
In particular, BPS has $m=n$ on this sequence of instances.
The support locations in this deterministic family vary with $n$.
The construction rules out a distribution-free sublinear bound, while Theorem~\ref{thm:candidate-logarithmic} concerns independent sampling from a fixed distribution satisfying its stated conditions.

\subsection{Uniform Slab Counting and Automatic Center Separation}
\label{app:candidate-counting}

Now let $\rvx_1,\ldots,\rvx_n$ be i.i.d.\ random vectors from a fixed distribution $P$, and consider the state class $\sZ_{n,\alpha}$ in~\eqref{eq:candidate-state-class}.
We control sample counts in $\sU(\vz)$ by covering it with slabs.
We first establish a counting bound that holds for all slabs simultaneously.
This uniformity allows both the centers and the covering slabs to depend on the observations.

Let $\sA_d$ be the class of Euclidean slabs, let
$N_n(\sA)=\sum_{i=1}^n\1_{\{\rvx_i\in \sA\}}$, and set
\begin{equation}
\Gamma_d(n,\delta)
=1+2(d+1)\log(2n+1)+\log(3/\delta).
\label{eq:candidate-vc-complexity}
\end{equation}
The following relative counting bound is uniform over slab locations, directions, and widths, including those chosen from the observations.

\begin{proposition}[Uniform slab counts]
\label{prop:candidate-slab-counts}
For i.i.d.\ random vectors from $P$ and every $\delta\in(0,1)$, with probability at least $1-\delta$,
\begin{equation}
N_n(\sA)\le2nP(\sA)+100\Gamma_d(n,\delta)
\qquad\text{for all }\sA\in\sA_d.
\label{eq:candidate-relative-count}
\end{equation}
\end{proposition}
\begin{proof}
Affine halfspaces in $\R^d$ have at most $(M+1)^{d+1}$ distinct traces on $M$ points.
A slab is the intersection of two halfspaces, so $\sA_d$ has at most $(M+1)^{2(d+1)}$ traces.
Introduce an independent sample $\rvx'_1,\ldots,\rvx'_n$, and let $\rc'$ count its observations in $\sA$.
If a slab selected from the original sample satisfies $\rc:=N_n(\sA)>2nP(\sA)+t$, then Markov's inequality, conditional on this sample and the selected slab, gives
\[
\mathbb P\{\rc'\le3\rc/4\mid \rvx_1,\ldots,\rvx_n\}
\ge1-\frac{4nP(\sA)}{3\rc}>\frac13.
\]
Hence
\[
\mathbb P\{\exists \sA:N_n(\sA)>2nP(\sA)+t\}
\le3\mathbb P\{\exists \sA:\rc>t,\ \rc'\le3\rc/4\}.
\]
Independently exchange $\rvx_i$ and $\rvx'_i$ with probability $1/2$ for every $i$, which preserves their joint distribution.
Conditional on the unordered pairs, a fixed trace of $\sA$ on the pooled observations has fixed total $S=\rc+\rc'$.
The difference $\rc-\rc'$ is a sum of $M_{\sA}$ independent symmetric signs, where $M_{\sA}\le S$ counts pairs with exactly one point in $\sA$.
The bound $\cosh u\le e^{u^2/2}$ therefore gives
\[
\mathbb P\{\rc-\rc'\ge x\mid\text{unordered pairs}\}
\le\exp\!\left(-\frac{x^2}{2M_{\sA}}\right)
\]
if $M_{\sA}>0$; a positive deviation is impossible otherwise.
On the event $\rc>t,\ \rc'\le3\rc/4$, we have $S>t$ and $\rc-\rc'\ge S/7$.
The conditional probability of this event for a fixed trace is at most $e^{-t/98}$.
A union bound over the pooled traces bounds the original failure probability by
$3(2n+1)^{2(d+1)}e^{-t/98}$.
Taking $t=100\Gamma_d(n,\delta)$ makes this at most $\delta$.
\end{proof}

Let $\rvx\sim P$. Suppose every unit projection has a density satisfying
\begin{equation}
\sup_{\lVert \vv\rVert =1}\lVert f_{\vv^\top\rvx}\rVert _\infty\le \kappa<\infty.
\label{eq:candidate-projection-density}
\end{equation}
A slab of total width $w$ then has probability at most $\kappa w$, uniformly over its location and direction.
We use this property to derive empirical center separation from D-locality and the occupancy condition.
Write $h_\alpha=\alpha/(8\kappa)$.
For $n$ such that $\alpha n\ge4$ and $100\Gamma_d(n,\delta)/n\le\alpha/4$, we show that the event in~\eqref{eq:candidate-relative-count} implies
\begin{equation}
\min_{a\ne b}\lVert \vmu_a-\vmu_b\rVert \ge h_\alpha
\qquad\text{for every }\vz\in\sZ_{n,\alpha}.
\label{eq:candidate-automatic-separation}
\end{equation}
Bounded projection densities exclude point atoms, so all observations are distinct almost surely.
If two centers coincided, then for each $i\in \sI_a$,
\[
g_i(b)
=\left(\frac1{n_a-1}+\frac1{n_b+1}\right)D_{ia}.
\]
D-locality would force all points in $\sI_a$ to equal that center, contradicting distinctness and $n_a\ge4$.

For distinct centers, put $h=\sqrt{\Delta_{ab}}$ and $\vv=(\vmu_b-\vmu_a)/h$.
By~\eqref{eq:candidate-distance-correction}, D-locality makes each assigned center a nearest center, so
$\vv^\top(\vx_i-\vmu_a)\le h/2$ for $i\in\sI_a$.
The nonnegative distances $h/2-\vv^\top(\vx_i-\vmu_a)$ average to $h/2$, because the cluster mean is $\vmu_a$.
At least $n_a/2$ samples therefore lie in the slab
\[
\sA=\{\vx:|\vv^\top(\vx-\vmu_a)|\le h/2\}.
\]
Its probability is at most $\kappa h$ by~\eqref{eq:candidate-projection-density}.
Applying~\eqref{eq:candidate-relative-count} gives
\[
\frac{\alpha}{2}\le\frac{N_n(\sA)}n
\le2\kappa h+\frac{100\Gamma_d(n,\delta)}n
\le2\kappa h+\frac{\alpha}{4},
\]
which proves $h\ge h_\alpha$.
All statements hold on a single event for every empirical state and every center pair.

On that event, define the auxiliary slabs
\[
\sS_{ab}(\vz)
:=\left\{\vx\in\R^d:
\left|\lVert \vx-\vmu_a\rVert ^2-\lVert \vx-\vmu_b\rVert ^2\right|
\le\frac{25p_{\max}}{n_{\min}}\right\},
\qquad a<b.
\]
For any $\vx\in\sU(\vz)$, its nearest and second-nearest centers identify a pair whose squared-distance difference is at most $25p_{\max}/n_{\min}$.
Therefore,
\begin{equation}
\sU(\vz)
\subseteq\bigcup_{1\le a<b\le k}\sS_{ab}(\vz).
\label{eq:candidate-slab-cover}
\end{equation}
Combining this cover with Proposition~\ref{prop:candidate-near-localization} gives
\[
m(\vz)\le N_n\!\left(\sU(\vz)\right)
\le\sum_{a<b}N_n\!\left(\sS_{ab}(\vz)\right).
\]
For $\vv=(\vmu_b-\vmu_a)/\sqrt{\Delta_{ab}}$,
\[
\lVert\vx-\vmu_a\rVert^2-\lVert\vx-\vmu_b\rVert^2
=2\sqrt{\Delta_{ab}}\,
\vv^\top\left(\vx-\frac{\vmu_a+\vmu_b}{2}\right).
\]
Thus $\sS_{ab}(\vz)$ is a slab around the equal-distance hyperplane, with geometric half-width $25p_{\max}/(2n_{\min}\sqrt{\Delta_{ab}})$.
Its total projection width is $25p_{\max}/(n_{\min}\sqrt{\Delta_{ab}})$, so its probability is at most $25\kappa p_{\max}/(n_{\min}\sqrt{\Delta_{ab}})$.
The squared sample diameter $p_{\max}$ is common to all empirical states, and $n_{\min}\ge\alpha n$ throughout $\sZ_{n,\alpha}$.
Summing~\eqref{eq:candidate-relative-count} over the $\binom{k}{2}$ slabs yields
\begin{equation}
\overline m_n\le\binom{k}{2}\left[
\frac{50\kappa p_{\max}}{\alpha h_\alpha}+100\Gamma_d(n,\delta)
\right]
\label{eq:candidate-uniform-count}
\end{equation}
with probability at least $1-\delta$, under the stated sample-size conditions.
The bound holds even though $p_{\max}$ and the empirical centers depend on the observations, because the counting event is uniform over all slabs.
For fixed $k,d,\alpha,P$ and $\delta$, its right-hand side is a term proportional to $p_{\max}$ plus $\mathcal O(\log n)$.

\subsection{The Two Distribution Classes}
\label{app:candidate-distributions}

\begin{proof}[Proof of Theorem~\ref{thm:candidate-logarithmic}]
First suppose $P$ is supported on a ball of radius $R$ centered at $\vx_0$ and $\lVert f\rVert _\infty<\infty$.
For a unit vector $\vv$, integration over $\vv^\perp$ gives the projection density
\[
f_{\vv^\top\rvx}(t)=\int_{\vv^\perp}f(t\vv+\vy)\,d\vy.
\]
Every nonempty slice of the support has $(d-1)$-dimensional volume at most
$V_{d-1}(R)$, the volume of a radius-$R$ ball in $\R^{d-1}$.
Consequently,~\eqref{eq:candidate-projection-density} holds with
$\kappa=V_{d-1}(R)\lVert f\rVert _\infty$, using $V_0(R)=1$ when $d=1$.
Any two samples are at distance at most $2R$, so $p_{\max}\le4R^2$.
For every fixed $\delta\in(0,1)$, the sample-size conditions of Appendix~\ref{app:candidate-counting} hold eventually, and~\eqref{eq:candidate-uniform-count} is $\mathcal O(\log n)$.
This proves the first case.

Next suppose $P$ is the fixed $k$-component Gaussian mixture in the theorem.
Positive definiteness and finiteness give constants
\[
\sigma_-^2=\min_{\ell\in[k]}\lambda_{\min}(\mSigma_\ell)>0,
\qquad
\sigma_+^2=\max_{\ell\in[k]}\lambda_{\max}(\mSigma_\ell)<\infty.
\]
For any unit $\vv$, the projection within component $\ell$ is a univariate Gaussian with variance $\vv^\top\mSigma_\ell \vv\ge\sigma_-^2$.
Each component density is bounded by $(\sqrt{2\pi}\sigma_-)^{-1}$, and so is their convex mixture.
Thus~\eqref{eq:candidate-projection-density} holds with this value of $\kappa$.

To bound the distance scale, write
$\rvx_i=\vtheta_{\rz_i}+\mSigma_{\rz_i}^{1/2}\rvw_i$,
where the $\rvw_i$ are independent standard Gaussian vectors, independent of the component labels $\rz_i\in[k]$, with $\mathbb P\{\rz_i=\ell\}=\pi_\ell$.
Here upright symbols denote random draws, while $\vx_i$ denotes a fixed realization.
For any fixed $\vx_0$, let $R_0=\max_{\ell\in[k]}\lVert\vtheta_\ell-\vx_0\rVert$.
Fix $\delta\in(0,1)$.
The Gaussian identity $\E e^{\lVert\rvw_i\rVert^2/4}=2^{d/2}$, Markov's inequality, and a union bound imply
\[
\mathbb P\!\left\{
\max_i\lVert\rvw_i\rVert^2>2d\log2+4\log(2n/\delta)
\right\}\le\delta/2.
\]
On the complementary event, the realized squared sample diameter satisfies
\begin{equation}
p_{\max}\le4\max_i\lVert\vx_i-\vx_0\rVert^2
\le4\left[R_0+\sigma_+\sqrt{2d\log2+4\log(2n/\delta)}\right]^2.
\label{eq:candidate-gmm-scale}
\end{equation}
Intersect this event with the uniform counting event at confidence $1-\delta/2$.
Applying~\eqref{eq:candidate-uniform-count} with $\Gamma_d(n,\delta/2)$ gives $\overline m_n=\mathcal O(\log n)$ on an event of probability at least $1-\delta$.
Since $\delta$ is arbitrary, this proves~\eqref{eq:candidate-logarithmic} for the second case.
The argument requires neither separation of component means nor recovery of the mixture-component labels by the empirical clustering.
\end{proof}

The supremum in~\eqref{eq:candidate-logarithmic} makes the result uniform over all qualifying states, including states selected adaptively from the observations.
For a D-local algorithmic output $\widehat{\vz}_n$ and any $C>0$,
\[
\mathbb P\{m(\widehat{\vz}_n)>C\log n\}
\le
\mathbb P\{\widehat{\vz}_n\notin\sZ_{n,\alpha}\}
+\mathbb P\{\overline m_n>C\log n\}.
\]
If the occupancy failure probability tends to zero, this proves
$m(\widehat{\vz}_n)=\mathcal O_{\mathbb P}(\log n)$.
The convention for an empty state class does not assume that a state in $\sZ_{n,\alpha}$ exists or is attained algorithmically.

\section{Experimental Setup}
\label{app:experimental-setup}

\paragraph{Datasets.}
The clustering-quality comparison uses the twelve datasets in Table~\ref{tab:experimental-datasets}. The dataset names in the table link to public benchmark source pages.
The five marked datasets are used for the descent-trajectory study and heuristic comparisons; Table~\ref{tab:heuristic-enhancement} reports the available heuristic results.
For each real dataset, $k$ is set to the number of reference classes, and the clustering algorithms receive only the feature vectors and $k$.
Pixel representations are scaled to $[0,1]$, while CIFAR-10 and CIFAR-100 use pre-extracted DINOv2 features.
20News uses the supplied TF--IDF representation reduced to 100 dimensions by singular value decomposition, and Sensorless uses the supplied measurement features.
All methods use the same representations, without additional normalization or dimensionality reduction.
The candidate-set scaling study uses random subsets of six larger datasets, with its sampling and aggregation protocol specified in Appendix~\ref{app:candidate-scaling-protocol}.

\paragraph{DINOv2 features.}
For CIFAR-10 and CIFAR-100, we use the pretrained DINOv2 ViT-B/14 backbone without register tokens (\texttt{dinov2\_vitb14}) from the \href{https://github.com/facebookresearch/dinov2/tree/7764ea0f912e53c92e82eb78a2a1631e92725fc8}{official implementation}, revision \texttt{7764ea0}.
We combine the original training and test sets and use the fine labels for CIFAR-100.
Each RGB image is resized by bicubic interpolation to a shorter-side length of 256 pixels, center-cropped to $224\times224$, converted to a tensor with pixel values in $[0,1]$, and normalized channelwise using ImageNet means $(0.485,0.456,0.406)$ and standard deviations $(0.229,0.224,0.225)$.
The 768-dimensional feature is the final CLS token after LayerNorm (\texttt{x\_norm\_clstoken}), with no subsequent samplewise $\ell_2$ normalization.
Feature extraction uses PyTorch \texttt{2.11.0+cu128} in float32, with autocast and TF32 disabled; the saved features are converted to double precision for clustering.
Feature extraction is performed before clustering and is excluded from the reported clustering runtimes.

\begin{table}[t]
\centering
\caption{Datasets used in the clustering-quality comparison. Here $n$, $d$, and $k$ denote the number of samples, feature dimension, and number of reference classes, respectively. Dataset names link to benchmark source pages; the listed dimensions describe the representations used here. Datasets marked with $\dagger$ are used for the descent-trajectory study and heuristic comparisons.}
\label{tab:experimental-datasets}
\small
\begin{tabular}{ccccc}
\toprule
Dataset & $n$ & $d$ & $k$ & Representation \\
\midrule
\href{https://vision.ucsd.edu/datasets/yale-face-database}{Yale}$^{\dagger}$        & 165    & 1,024 & 15  & Pixels \\
\href{https://www.cl.cam.ac.uk/research/dtg/attarchive/facedatabase.html}{ORL}$^{\dagger}$         & 400    & 1,024 & 40  & Pixels \\
\href{https://www.csie.ntu.edu.tw/~cjlin/libsvmtools/datasets/multiclass.html#usps}{USPS-1000}$^{\dagger}$   & 1,000  & 256   & 10  & Pixels \\
\href{https://www.cs.columbia.edu/CAVE/software/softlib/coil-20.php}{COIL20}$^{\dagger}$      & 1,440  & 1,024 & 20  & Pixels \\
\href{https://huggingface.co/datasets/ylecun/mnist}{MNIST-2000}$^{\dagger}$  & 2,000  & 784   & 10  & Pixels \\
\href{https://www.cs.columbia.edu/CAVE/software/softlib/coil-100.php}{COIL100}                & 7,200  & 1,024 & 100 & Pixels \\
\href{http://qwone.com/~jason/20Newsgroups/}{20News}                 & 18,282 & 100   & 20  & TF--IDF + SVD \\
\href{https://doi.org/10.24432/C5VP5F}{Sensorless}             & 58,509 & 48    & 11  & Measurements \\
\href{https://www.cs.toronto.edu/~kriz/cifar.html}{CIFAR-10}               & 60,000 & 768   & 10  & DINOv2 \\
\href{https://www.cs.toronto.edu/~kriz/cifar.html}{CIFAR-100}              & 60,000 & 768   & 100 & DINOv2 \\
\href{https://huggingface.co/datasets/ylecun/mnist}{MNIST}                  & 70,000 & 784   & 10  & Pixels \\
\href{https://github.com/zalandoresearch/fashion-mnist}{Fashion-MNIST}          & 70,000 & 784   & 10  & Pixels \\
\bottomrule
\end{tabular}
\end{table}

\paragraph{Comparison methods.}
We organize the evaluated methods into two categories according to their update mechanisms.

\noindent\textbf{Sample-reassignment-based methods.}
These methods update sample memberships through batch assignments or one- and two-point reassignments.
\begin{itemize}
    \item \textbf{Lloyd} \citep{lloyd1982least}: Alternates between assigning samples to their nearest centroids and recomputing cluster means until the assignments stabilize.
    \item \textbf{lite-KM} \citep{cai2011litekmeans}: A vectorized MATLAB implementation of Lloyd's algorithm, included as an efficient baseline for batch reassignment.
    \item \textbf{Hartigan} \citep{telgarsky2010hartigan}: Sequentially scans samples and moves each sample to the destination with the largest positive gain after centroid recomputation, immediately updating the affected centroids.
    \item \textbf{D-LO} \citep{li2025modified}: Alternates Lloyd iterations with one-point local search, accepting the first improving feasible reassignment and then resuming Lloyd iterations.
    \item \textbf{Min-D-LO} \citep{li2025modified}: Uses the same framework as D-LO but selects the one-point reassignment with the largest objective decrease in each local-search step.
    \item \textbf{BPS-2PLS (ours)}: After initial Lloyd convergence, alternates the best-improvement one-point descent of Min-D-LO with exact two-point search, maintaining cluster statistics and distance columns through local updates. Lloyd is not restarted between local moves. Safe boundary-point screening reduces the candidate set while preserving every improving two-point move (Algorithm~\ref{alg:bps2pls}).
\end{itemize}

\noindent\textbf{Heuristic methods.}
These methods explore center replacements or cluster repartitions, which can change many sample memberships at once.
\begin{itemize}
    \item \textbf{LocalSearch} \citep{lattanzi2019better}: Samples a data point with probability proportional to its squared distance to the nearest center and accepts the best improving replacement of one current center by this point.
    \item \textbf{MultiSwap} \citep{beretta2023multiswap}: Extends center swaps to multiple sampled points. Our implementation retains the best $k$ centers from the pool of current centers and sampled points, and applies Lloyd refinement after an improving exchange.
    \item \textbf{J-Means} \citep{hansen2001jmeans}: Escapes Lloyd fixed points by relocating a centroid to a data sample. Our implementation scores candidate jumps after one assignment-and-centroid update and follows the best improving jump with Lloyd refinement.
    \item \textbf{MergeSplit} \citep{nielsen2014further}: Merges two clusters and repartitions their union using repeated $2$-means trials. An improving split is accepted and followed by global Lloyd refinement.
\end{itemize}
We evaluate the four heuristic methods with and without BPS-2PLS; the refinement insertion points are described in Appendix~\ref{app:heuristic-refinements}.

\paragraph{Heuristic search budgets.}
LocalSearch uses $20$ sampled replacement attempts.
MultiSwap samples two candidate points per outer step and runs for at most $10$ steps.
MergeSplit uses five random $2$-means trials per cluster pair, while J-Means stops when no improving jump is found.
Each baseline and its enhanced variants use the same method-specific budget and acceptance rule; refinement time is included in the reported runtime.
Hartigan on COIL100, CIFAR-10, CIFAR-100, MNIST, and Fashion-MNIST, and all J-Means variants on COIL20 and MNIST-2000, were omitted because of excessive runtime; these runs are marked with dashes.

\paragraph{Initialization.}
For the quality and enhancement comparisons, we generate ten independent $k$-means++ initial center sets per dataset and share them across methods.
Each randomized outer method and its enhanced variants use paired random seeds and the same outer search budgets.

\paragraph{Evaluation metrics.}
We report squared-error loss, clustering accuracy (ACC), and normalized mutual information (NMI).
Loss is recomputed as $\sum_{i=1}^n\lVert\vx_i-\vmu_{z_i}\rVert^2$ from each method's returned assignments $\vz$ and centers $\vmu_1,\ldots,\vmu_k$.
ACC uses an optimal matching between cluster labels and reference labels, and NMI is defined as $2\mathrm I(\vz;\vy)/(\mathrm H(\vz)+\mathrm H(\vy))$, where $\vy$ is the reference-label vector, and $\mathrm I$ and $\mathrm H$ denote mutual information and entropy computed from empirical label frequencies.
Across the ten initializations, we report the best value, mean, and standard deviation of each quality metric.
Best means minimum loss or maximum ACC/NMI; these extrema are selected independently and need not come from the same initialization.

\paragraph{Timing.}
For the quality and enhancement comparisons, each fixed initialization is run five times, resetting any algorithm-specific random stream before each repetition.
We first average the five runtimes for each initialization and then report the mean and standard deviation across the ten resulting averages; the repetitions do not constitute additional independent initializations.
Quality metrics are recorded once per initialization.
Each method is warmed up before measurement, and method order is rotated across repetitions.
Reported runtimes include all algorithmic work, including distance calculations and screening, but exclude data loading, shared initial-center generation, quality evaluation, and result serialization.
Phase diagnostics and MATLAB profiling are collected in separate runs to avoid adding their overhead to the ordinary timing measurements.

\paragraph{Implementation.}
All methods are implemented in MATLAB using double-precision arithmetic.
We use the lazy policy in Appendix~\ref{app:distance-policies}: after screening, distances are evaluated only between retained samples, in blocks of at most 256 rows, and each block is discarded after its pairs are processed. A distance is reused across destination choices; no dataset-wide sample-pair cache is precomputed or retained across rounds. Remaining pairs whose certificates cannot exceed the best gain found so far may also be skipped, which can reduce work below the worst-case bound. Reported total runtimes include all distance calculations and state updates; the Search phase in Appendix~\ref{app:runtime-phases} includes candidate-pair distance computation.
The clustering-quality and heuristic-enhancement experiments were run in MATLAB R2026a on an Intel Xeon Gold 6248R CPU at 3.00 GHz, with memory operating at 2666 MHz. The separate runtime-phase study is described in Appendix~\ref{app:runtime-phases}.

\section{Supplementary Experimental Discussion}
\label{app:additional-experiments}

\subsection{WCSS Improvements in Larger Clusters}
A plausible explanation is the diminishing influence of a few samples on well-populated clusters.
For a sample with $s=z_i$ and destination $a$, rewriting~\eqref{eq:one-point-gain} gives $g_i(a)=D_{is}-D_{ia}+D_{is}/(n_s-1)+D_{ia}/(n_a+1)$.
At a Lloyd fixed point, $D_{is}-D_{ia}\le0$; if residual distances remain on comparable scales, the two positive corrections shrink as the affected clusters grow, restricting improving relocations to samples with small assignment margins.
Two-point moves similarly induce smaller centroid changes in larger clusters.
Unless these moves trigger substantial subsequent descent, their gains can remain small relative to WCSS summed over all samples.
This interpretation concerns the one- and two-point neighborhoods; larger coordinated rearrangements may still yield improvements.

\subsection{Heuristic Refinement Insertion Points}
\label{app:heuristic-refinements}
For each baseline, the \emph{Post} variant applies BPS-2PLS to its final clustering.
This terminal refinement measures the residual improvement available after the original heuristic stops.

For MultiSwap, J-Means, and MergeSplit, we also consider an \emph{Inner} variant that embeds BPS-2PLS within the search.
In MultiSwap and J-Means, BPS-2PLS replaces the initial Lloyd refinement and the Lloyd refinement following each accepted outer move; candidate generation and scoring rules remain unchanged.
In MergeSplit, BPS-2PLS is applied with $k=2$ to every trial split of the selected cluster pair, before choosing the best trial and deciding whether to accept it; global Lloyd refinement is retained.
LocalSearch is evaluated only with terminal enhancement.
The comparison separates improvements available at termination from changes produced by refining intermediate states.
Terminal refinement cannot increase WCSS for a given baseline output, whereas internal refinement changes the subsequent search trajectory and need not outperform the baseline on every initialization.

\subsection{Candidate-Set Scaling Protocol and Interpretation}
\label{app:candidate-scaling-protocol}
We study how the retained count $m$ changes with $n$ on 20News, Sensorless, CIFAR-10, CIFAR-100, MNIST, and Fashion-MNIST.
For each dataset, we sample subsets uniformly without replacement at $n\in\{200,500,1000,2000,5000,10000,20000\}$; 20News is evaluated only up to $n=10000$.
At each size, we generate five independent subsets and run three $k$-means++ initializations per subset.
For each subset, we first average $m$ over the three initializations.
Figure~\ref{fig:candidate-scaling} then reports the mean and standard deviation of the five subset averages.

\paragraph{Implications for the theoretical bound.}
These observations raise the possibility that the $\mathcal O_{\mathbb P}(\log n)$ upper bound in Theorem~\ref{thm:candidate-logarithmic} is conservative for states reached by the algorithm.
The theorem controls all D-local states satisfying the occupancy condition through global diameter bounds and uniform slab counting, which may overestimate the number of candidates at a typical computed state.
It is therefore natural to investigate whether a sharper analysis can establish $m=\mathcal O_{\mathbb P}(1)$ for algorithmic outputs under suitable distributional and occupancy assumptions.
This remains a conjecture: averages from finite-dataset subsampling cannot establish an asymptotic rate or reliably distinguish constant behavior from slow logarithmic growth.

\subsection{Additional Clustering Quality for Sample-Reassignment Methods}
\label{app:quality-acc-nmi}
Tables~\ref{tab:family-acc} and~\ref{tab:family-nmi} complement the WCSS and runtime comparison in Table~\ref{tab:quality-cost} with ACC and NMI.
Both metrics are expressed as percentages and summarized by their means and standard deviations across the same ten initializations; timing repetitions do not add independent quality measurements.
The Yale results in these tables use 165 samples and 15 classes.
The comparisons use the available results only, without imputing missing methods.

\begin{table}[htbp]
\centering
\caption{ACC (\%) for the sample-reassignment methods over ten shared initializations. Entries are mean $\pm$ standard deviation; higher is better. Bold marks the highest available mean on each dataset, including ties at the displayed precision. Dashes mark Hartigan runs omitted because of excessive runtime.}
\label{tab:family-acc}
\footnotesize
\setlength{\tabcolsep}{2.25pt}
\begin{tabular}{@{}l*{6}{c}@{}}
\toprule
Dataset & Lloyd & lite-KM & Hartigan & D-LO & Min-D-LO & BPS-2PLS \\
\midrule
Yale & $38.67\pm2.35$ & $38.67\pm2.35$ & $41.39\pm3.82$ & $41.03\pm2.46$ & $41.03\pm2.46$ & $\mathbf{41.64}\pm2.16$ \\
ORL & $48.85\pm4.12$ & $48.85\pm4.12$ & $54.73\pm3.09$ & $53.88\pm4.36$ & $53.88\pm4.36$ & $\mathbf{54.88}\pm3.09$ \\
USPS-1000 & $45.69\pm3.06$ & $45.69\pm3.06$ & $\mathbf{47.67}\pm2.34$ & $47.59\pm2.74$ & $47.59\pm2.74$ & $47.48\pm2.54$ \\
COIL20 & $60.76\pm5.20$ & $60.76\pm5.20$ & $61.69\pm4.33$ & $61.88\pm5.64$ & $61.88\pm5.64$ & $\mathbf{62.03}\pm5.46$ \\
MNIST-2000 & $54.78\pm3.31$ & $54.78\pm3.31$ & $54.29\pm3.50$ & $\mathbf{56.12}\pm4.21$ & $\mathbf{56.12}\pm4.21$ & $56.07\pm4.13$ \\
COIL100 & $46.11\pm1.68$ & $46.11\pm1.68$ & --- & $46.72\pm1.64$ & $46.72\pm1.64$ & $\mathbf{46.99}\pm1.80$ \\
20News & $37.39\pm2.32$ & $37.39\pm2.32$ & $\mathbf{37.46}\pm1.61$ & $37.35\pm2.42$ & $37.36\pm2.43$ & $37.34\pm2.43$ \\
Sensorless & $21.71\pm2.35$ & $21.71\pm2.35$ & $\mathbf{22.40}\pm1.93$ & $21.72\pm2.35$ & $21.72\pm2.35$ & $21.72\pm2.36$ \\
CIFAR-10 & $\mathbf{78.23}\pm5.93$ & $\mathbf{78.23}\pm5.93$ & --- & $78.22\pm5.93$ & $78.22\pm5.93$ & $78.22\pm5.93$ \\
CIFAR-100 & $66.09\pm2.16$ & $66.09\pm2.16$ & --- & $\mathbf{66.23}\pm2.23$ & $\mathbf{66.23}\pm2.23$ & $\mathbf{66.23}\pm2.23$ \\
MNIST & $\mathbf{55.50}\pm2.78$ & $\mathbf{55.50}\pm2.78$ & --- & $55.43\pm2.69$ & $55.43\pm2.69$ & $55.43\pm2.69$ \\
Fashion-MNIST & $51.67\pm2.97$ & $51.67\pm2.97$ & --- & $51.67\pm2.97$ & $\mathbf{51.68}\pm2.97$ & $\mathbf{51.68}\pm2.97$ \\
\bottomrule
\end{tabular}
\end{table}

\begin{table}[htbp]
\centering
\caption{NMI (\%) for the sample-reassignment methods over ten shared initializations. Entries are mean $\pm$ standard deviation; higher is better. Bold marks the highest available mean on each dataset, including ties at the displayed precision. Dashes mark Hartigan runs omitted because of excessive runtime.}
\label{tab:family-nmi}
\footnotesize
\setlength{\tabcolsep}{2.25pt}
\begin{tabular}{@{}l*{6}{c}@{}}
\toprule
Dataset & Lloyd & lite-KM & Hartigan & D-LO & Min-D-LO & BPS-2PLS \\
\midrule
Yale & $47.06\pm2.14$ & $47.06\pm2.14$ & $50.32\pm2.52$ & $49.95\pm1.84$ & $49.95\pm1.84$ & $\mathbf{50.56}\pm2.74$ \\
ORL & $72.92\pm2.10$ & $72.92\pm2.10$ & $77.08\pm1.23$ & $76.63\pm1.92$ & $76.63\pm1.92$ & $\mathbf{77.29}\pm1.87$ \\
USPS-1000 & $44.10\pm1.22$ & $44.10\pm1.22$ & $\mathbf{46.94}\pm2.76$ & $46.12\pm2.30$ & $46.12\pm2.30$ & $46.06\pm2.00$ \\
COIL20 & $76.51\pm1.65$ & $76.51\pm1.65$ & $76.77\pm1.62$ & $77.28\pm2.11$ & $77.28\pm2.11$ & $\mathbf{77.38}\pm2.04$ \\
MNIST-2000 & $48.75\pm2.65$ & $48.75\pm2.65$ & $49.01\pm2.76$ & $\mathbf{49.49}\pm2.80$ & $\mathbf{49.49}\pm2.80$ & $49.45\pm2.76$ \\
COIL100 & $76.26\pm0.34$ & $76.26\pm0.34$ & --- & $76.74\pm0.36$ & $76.74\pm0.36$ & $\mathbf{76.92}\pm0.38$ \\
20News & $37.48\pm1.50$ & $37.48\pm1.50$ & $37.45\pm1.45$ & $\mathbf{37.56}\pm1.57$ & $\mathbf{37.56}\pm1.57$ & $\mathbf{37.56}\pm1.57$ \\
Sensorless & $16.12\pm1.45$ & $16.12\pm1.45$ & $\mathbf{16.35}\pm1.44$ & $16.13\pm1.45$ & $16.13\pm1.45$ & $16.13\pm1.47$ \\
CIFAR-10 & $\mathbf{85.61}\pm1.45$ & $\mathbf{85.61}\pm1.45$ & --- & $85.60\pm1.45$ & $85.60\pm1.45$ & $85.60\pm1.45$ \\
CIFAR-100 & $\mathbf{79.40}\pm0.42$ & $\mathbf{79.40}\pm0.42$ & --- & $79.37\pm0.40$ & $79.37\pm0.40$ & $79.37\pm0.40$ \\
MNIST & $49.93\pm1.14$ & $49.93\pm1.14$ & --- & $\mathbf{49.99}\pm1.19$ & $\mathbf{49.99}\pm1.19$ & $\mathbf{49.99}\pm1.19$ \\
Fashion-MNIST & $\mathbf{51.72}\pm1.54$ & $\mathbf{51.72}\pm1.54$ & --- & $\mathbf{51.72}\pm1.55$ & $\mathbf{51.72}\pm1.55$ & $51.71\pm1.55$ \\
\bottomrule
\end{tabular}
\end{table}

\paragraph{Improvements and their limits.}
BPS-2PLS improves both mean ACC and mean NMI over Lloyd on Yale, ORL, USPS-1000, COIL20, MNIST-2000, COIL100, and Sensorless.
The ACC increases are $2.97$ and $6.03$ percentage points on Yale and ORL, respectively, and the corresponding NMI increases are $3.50$ and $4.37$ percentage points.
BPS-2PLS has the highest available mean of both metrics on Yale, ORL, COIL20, and COIL100.
However, lower WCSS need not yield better agreement with reference labels: on USPS-1000, Hartigan attains $47.67\%$ ACC and $46.94\%$ NMI, compared with $47.48\%$ and $46.06\%$ for BPS-2PLS, despite the lower WCSS of BPS-2PLS in Table~\ref{tab:quality-cost}.

On the six larger datasets from 20News through Fashion-MNIST, the absolute changes in mean ACC and NMI between BPS-2PLS and D-LO are below $0.01$ percentage points before rounding.
Thus, the additional WCSS reduction on these datasets is accompanied by nearly unchanged label agreement.
The local-optimality guarantee concerns WCSS and does not imply an improvement in ACC or NMI.

\subsection{Additional Clustering Quality for Heuristic Enhancement}
\label{app:enhancement-acc-nmi}
Tables~\ref{tab:enhancement-acc} and~\ref{tab:enhancement-nmi} supplement Table~\ref{tab:heuristic-enhancement} with ACC and NMI on the five datasets, using the same ten shared initializations.
The \emph{Post} and \emph{Inner} variants are defined in Appendix~\ref{app:heuristic-refinements}.
The Min-D-LO reference block uses Min-D-LO as Baseline and BPS-2PLS as Inner. J-Means was omitted on COIL20 and MNIST-2000 because of excessive runtime.

\begin{table}[htbp]
\centering
\caption{ACC (\%) for heuristic enhancement: mean $\pm$ standard deviation over ten initializations. Bold marks the highest mean within each method and dataset, including ties at the displayed precision. Dashes mark J-Means runs omitted because of excessive runtime. For Min-D-LO, Inner reports BPS-2PLS.}
\label{tab:enhancement-acc}
\footnotesize
\setlength{\tabcolsep}{2.25pt}
\begin{tabular}{@{}ll*{5}{c}@{}}
\toprule
Method & Variant & Yale & ORL & USPS-1000 & COIL20 & MNIST-2000 \\
\midrule
Min-D-LO & Baseline & $41.03\pm2.46$ & $53.88\pm4.36$ & $\mathbf{47.59}\pm2.74$ & $61.88\pm5.64$ & $\mathbf{56.12}\pm4.21$ \\
 & Inner & $\mathbf{41.64}\pm2.16$ & $\mathbf{54.88}\pm3.09$ & $47.48\pm2.54$ & $\mathbf{62.03}\pm5.46$ & $56.07\pm4.13$ \\
\midrule
LocalSearch & Baseline & $38.61\pm3.16$ & $53.00\pm2.39$ & $37.59\pm2.64$ & $57.72\pm3.64$ & $45.31\pm3.42$ \\
 & Post & $\mathbf{43.15}\pm4.33$ & $\mathbf{57.55}\pm2.79$ & $\mathbf{47.86}\pm2.07$ & $\mathbf{63.01}\pm4.79$ & $\mathbf{53.05}\pm3.99$ \\
\midrule
MultiSwap & Baseline & $36.97\pm3.64$ & $52.10\pm0.78$ & $46.14\pm2.92$ & $58.34\pm4.89$ & $\mathbf{54.65}\pm5.78$ \\
 & Post & $40.61\pm4.13$ & $\mathbf{56.33}\pm1.63$ & $\mathbf{48.12}\pm2.92$ & $\mathbf{59.63}\pm4.25$ & $54.62\pm5.33$ \\
 & Inner & $\mathbf{40.79}\pm4.23$ & $55.45\pm1.74$ & $\mathbf{48.12}\pm2.92$ & $\mathbf{59.63}\pm4.25$ & $54.62\pm5.33$ \\
\midrule
J-Means & Baseline & $40.24\pm2.49$ & $57.18\pm2.22$ & $46.06\pm3.00$ & --- & --- \\
 & Post & $\mathbf{42.73}\pm2.93$ & $58.05\pm2.17$ & $48.02\pm3.01$ & --- & --- \\
 & Inner & $42.42\pm3.32$ & $\mathbf{58.20}\pm2.02$ & $\mathbf{48.12}\pm2.92$ & --- & --- \\
\midrule
MergeSplit & Baseline & $38.67\pm3.49$ & $58.75\pm1.93$ & $48.27\pm2.90$ & $68.65\pm4.22$ & $54.72\pm5.24$ \\
 & Post & $39.94\pm2.79$ & $\mathbf{58.90}\pm1.87$ & $48.31\pm3.05$ & $68.72\pm4.21$ & $54.70\pm5.27$ \\
 & Inner & $\mathbf{42.18}\pm3.89$ & $58.33\pm1.87$ & $\mathbf{48.46}\pm3.24$ & $\mathbf{68.76}\pm3.76$ & $\mathbf{54.82}\pm5.07$ \\
\bottomrule
\end{tabular}
\end{table}

\begin{table}[htbp]
\centering
\caption{NMI (\%) for heuristic enhancement: mean $\pm$ standard deviation over ten initializations. Bold marks the highest mean within each method and dataset, including ties at the displayed precision. Dashes mark J-Means runs omitted because of excessive runtime. For Min-D-LO, Inner reports BPS-2PLS.}
\label{tab:enhancement-nmi}
\footnotesize
\setlength{\tabcolsep}{2.25pt}
\begin{tabular}{@{}ll*{5}{c}@{}}
\toprule
Method & Variant & Yale & ORL & USPS-1000 & COIL20 & MNIST-2000 \\
\midrule
Min-D-LO & Baseline & $49.95\pm1.84$ & $76.63\pm1.92$ & $\mathbf{46.12}\pm2.30$ & $77.28\pm2.11$ & $\mathbf{49.49}\pm2.80$ \\
 & Inner & $\mathbf{50.56}\pm2.74$ & $\mathbf{77.29}\pm1.87$ & $46.06\pm2.00$ & $\mathbf{77.38}\pm2.04$ & $49.45\pm2.76$ \\
\midrule
LocalSearch & Baseline & $45.08\pm1.87$ & $72.18\pm1.06$ & $33.05\pm1.68$ & $70.72\pm1.04$ & $35.68\pm3.67$ \\
 & Post & $\mathbf{50.86}\pm2.88$ & $\mathbf{77.83}\pm1.26$ & $\mathbf{46.60}\pm1.99$ & $\mathbf{77.46}\pm1.10$ & $\mathbf{47.72}\pm2.19$ \\
\midrule
MultiSwap & Baseline & $45.47\pm2.79$ & $74.48\pm0.62$ & $44.44\pm2.65$ & $75.72\pm1.50$ & $49.44\pm2.84$ \\
 & Post & $49.44\pm2.87$ & $\mathbf{77.68}\pm0.90$ & $\mathbf{46.65}\pm2.67$ & $\mathbf{76.61}\pm1.19$ & $\mathbf{49.79}\pm2.52$ \\
 & Inner & $\mathbf{49.61}\pm2.93$ & $77.53\pm1.15$ & $\mathbf{46.65}\pm2.67$ & $\mathbf{76.61}\pm1.19$ & $\mathbf{49.79}\pm2.52$ \\
\midrule
J-Means & Baseline & $47.37\pm2.12$ & $78.13\pm0.73$ & $44.47\pm2.63$ & --- & --- \\
 & Post & $\mathbf{50.95}\pm2.00$ & $\mathbf{78.87}\pm0.85$ & $\mathbf{46.72}\pm2.61$ & --- & --- \\
 & Inner & $50.64\pm2.55$ & $78.60\pm0.80$ & $46.65\pm2.67$ & --- & --- \\
\midrule
MergeSplit & Baseline & $48.04\pm2.57$ & $78.66\pm1.03$ & $46.73\pm2.49$ & $79.46\pm1.26$ & $49.73\pm2.59$ \\
 & Post & $49.56\pm2.34$ & $\mathbf{78.83}\pm0.77$ & $\mathbf{46.97}\pm2.63$ & $79.59\pm1.32$ & $49.77\pm2.61$ \\
 & Inner & $\mathbf{50.74}\pm2.65$ & $78.65\pm0.54$ & $46.66\pm2.79$ & $\mathbf{79.63}\pm0.97$ & $\mathbf{49.90}\pm2.40$ \\
\bottomrule
\end{tabular}
\end{table}

\paragraph{Terminal refinement.}
Among the four heuristic methods, Post improves mean ACC in sixteen of eighteen available method--dataset comparisons and mean NMI in all eighteen.
ACC changes range from $-0.03$ to $10.27$ percentage points, and NMI increases range from $0.04$ to $13.55$ percentage points. The two ACC decreases occur for MultiSwap and MergeSplit on MNIST-2000.
The largest gains occur for LocalSearch on USPS-1000, where ACC rises from $37.59\%$ to $47.86\%$ and NMI from $33.05\%$ to $46.60\%$.
These results show that the WCSS improvements reported in the main text can also improve agreement with reference labels, although label agreement is not part of the refinement objective.

\paragraph{Internal refinement.}
Among the three heuristic methods with an Inner variant, Inner improves mean ACC in eleven of thirteen comparisons with the corresponding baseline and likewise improves mean NMI in eleven of thirteen comparisons.
Its relative performance against Post depends on the dataset and metric.
For example, MergeSplit--Inner on Yale attains $42.18\%$ ACC and $50.74\%$ NMI, compared with $39.94\%$ and $49.56\%$ for Post.
On ORL, however, MergeSplit--Inner has lower mean ACC than both the baseline and Post ($58.33\%$, versus $58.75\%$ and $58.90\%$), even though it has the lowest mean WCSS among the three variants.
For MergeSplit on USPS-1000, Inner has the highest ACC but lower NMI than the other variants.
On COIL20, MergeSplit--Inner has higher mean ACC and NMI than Post despite a higher mean WCSS.
These differences reflect changes in the search trajectory and the distinction between optimizing WCSS and matching reference classes.
The tables provide descriptive comparisons across initializations; they do not establish statistical significance.

\subsection{Runtime by Algorithmic Phase}
\label{app:runtime-phases}
We measure the runtime composition of BPS-2PLS on all twelve datasets in Table~\ref{tab:experimental-datasets} with screening enabled.
This separate timing study was run in MATLAB R2026b.

\paragraph{Measurement protocol.}
We use three $k$-means++ initializations per dataset and five timed repetitions per initialization after warm-up.
All four phase timers and the total wall-clock time are collected from the same instrumented run, with MATLAB's function-level Profiler disabled.
We first average repetitions within each initialization, then compute means and sample standard deviations across the three initialization averages.
A phase's percentage is its mean runtime divided by the mean total runtime.
These measurements characterize this timing batch and are not pooled with the ten-initialization quality experiments.

\paragraph{Phase definitions.}
\emph{Descent} includes Lloyd iterations and one-point local descent throughout the run.
\emph{Screen} constructs and evaluates the screening bounds, and \emph{Search} performs the exact search over the retained sample pairs, including their distance calculations.
\emph{Other} includes initialization of the clustering state, updates for accepted two-point moves, and remaining overhead.
The four phases partition the measured total; pair-distance computation is included in Search and is not counted again in Other.

\begin{table}[htbp]
\centering
\caption{Runtime by phase with screening enabled. Times are in seconds, reported as mean $\pm$ standard deviation across three initializations after averaging five repetitions within each initialization. Parentheses give each phase's percentage of the mean total time; percentages below $0.1\%$ are shown to two decimal places. Bold marks the largest mean phase time for each dataset.}
\label{tab:runtime-phases}
\footnotesize
\setlength{\tabcolsep}{3pt}
\begin{tabular}{@{}l*{5}{c}@{}}
\toprule
Dataset & Total & Descent & Screen & Search & Other \\
\midrule
Yale & $0.268\pm0.013$ & \makecell{$\mathbf{0.149}\pm0.025$\\$(55.5\%)$} & \makecell{$0.077\pm0.016$\\$(28.9\%)$} & \makecell{$0.036\pm0.015$\\$(13.5\%)$} & \makecell{$0.006\pm0.002$\\$(2.1\%)$} \\[3pt]
ORL & $2.090\pm0.538$ & \makecell{$0.695\pm0.091$\\$(33.3\%)$} & \makecell{$\mathbf{1.269}\pm0.442$\\$(60.7\%)$} & \makecell{$0.100\pm0.054$\\$(4.8\%)$} & \makecell{$0.026\pm0.010$\\$(1.2\%)$} \\[3pt]
USPS-1000 & $1.032\pm0.582$ & \makecell{$\mathbf{0.923}\pm0.582$\\$(89.4\%)$} & \makecell{$0.080\pm0.043$\\$(7.7\%)$} & \makecell{$0.021\pm0.011$\\$(2.0\%)$} & \makecell{$0.009\pm0.005$\\$(0.9\%)$} \\[3pt]
COIL20 & $1.256\pm0.546$ & \makecell{$\mathbf{1.055}\pm0.584$\\$(84.0\%)$} & \makecell{$0.176\pm0.062$\\$(14.0\%)$} & \makecell{$0.005\pm0.001$\\$(0.4\%)$} & \makecell{$0.020\pm0.008$\\$(1.6\%)$} \\[3pt]
MNIST-2000 & $1.681\pm0.565$ & \makecell{$\mathbf{1.590}\pm0.544$\\$(94.6\%)$} & \makecell{$0.059\pm0.025$\\$(3.5\%)$} & \makecell{$0.015\pm0.006$\\$(0.9\%)$} & \makecell{$0.018\pm0.010$\\$(1.1\%)$} \\[3pt]
COIL100 & $76.852\pm6.224$ & \makecell{$\mathbf{40.142}\pm3.222$\\$(52.2\%)$} & \makecell{$34.024\pm4.556$\\$(44.3\%)$} & \makecell{$2.027\pm0.842$\\$(2.6\%)$} & \makecell{$0.660\pm0.087$\\$(0.9\%)$} \\[3pt]
20News & $29.293\pm17.179$ & \makecell{$\mathbf{26.952}\pm16.185$\\$(92.0\%)$} & \makecell{$2.024\pm0.917$\\$(6.9\%)$} & \makecell{$0.254\pm0.112$\\$(0.9\%)$} & \makecell{$0.062\pm0.027$\\$(0.2\%)$} \\[3pt]
Sensorless & $3.247\pm0.409$ & \makecell{$\mathbf{2.068}\pm0.761$\\$(63.7\%)$} & \makecell{$1.145\pm0.508$\\$(35.3\%)$} & \makecell{$0.001\pm0.001$\\$(0.04\%)$} & \makecell{$0.034\pm0.015$\\$(1.0\%)$} \\[3pt]
CIFAR-10 & $13.215\pm6.050$ & \makecell{$\mathbf{12.228}\pm6.059$\\$(92.5\%)$} & \makecell{$0.804\pm0.011$\\$(6.1\%)$} & \makecell{$0.004\pm0.002$\\$(0.03\%)$} & \makecell{$0.179\pm0.007$\\$(1.4\%)$} \\[3pt]
CIFAR-100 & $222.695\pm42.560$ & \makecell{$\mathbf{145.161}\pm23.062$\\$(65.2\%)$} & \makecell{$69.356\pm54.717$\\$(31.1\%)$} & \makecell{$6.569\pm5.205$\\$(2.9\%)$} & \makecell{$1.609\pm1.292$\\$(0.7\%)$} \\[3pt]
MNIST & $59.946\pm40.294$ & \makecell{$\mathbf{57.949}\pm40.096$\\$(96.7\%)$} & \makecell{$1.608\pm0.543$\\$(2.7\%)$} & \makecell{$0.012\pm0.004$\\$(0.02\%)$} & \makecell{$0.376\pm0.152$\\$(0.6\%)$} \\[3pt]
Fashion-MNIST & $34.921\pm8.596$ & \makecell{$\mathbf{32.518}\pm8.565$\\$(93.1\%)$} & \makecell{$1.916\pm0.032$\\$(5.5\%)$} & \makecell{$0.005\pm0.002$\\$(0.01\%)$} & \makecell{$0.482\pm0.017$\\$(1.4\%)$} \\[3pt]
\bottomrule
\end{tabular}
\end{table}

\paragraph{Dominant costs.}
Descent is the largest phase on eleven of the twelve datasets, accounting for $52.2\%$--$96.7\%$ of total runtime on those datasets.
ORL is the exception, where screening accounts for $60.7\%$ of runtime compared with $33.3\%$ for descent.
Screening is also substantial on COIL100, Sensorless, CIFAR-100, and Yale, where it accounts for $44.3\%$, $35.3\%$, $31.1\%$, and $28.9\%$ of runtime, respectively.
Exact pair search remains relatively small: it accounts for at most $13.5\%$ on Yale and no more than $4.8\%$ on any other dataset; on Sensorless, CIFAR-10, MNIST, and Fashion-MNIST it is below $0.05\%$.
Other overhead remains at or below $2.1\%$ across all datasets.
These phase shares suggest that further runtime reductions on the evaluated workloads depend primarily on accelerating descent and screening rather than the remaining exact pair search.
Because the timing study contains only screening-enabled runs, the breakdown characterizes runtime composition rather than the net speedup attributable to screening.

\end{document}